\documentclass[11pt]{article}

\usepackage{amssymb,amsfonts,amsmath,amsthm,amscd,dsfont,mathrsfs}
\usepackage{graphicx,float,psfrag,epsfig}
\usepackage{wrapfig}
\usepackage{algorithm,algorithmic,hyperref, fullpage}

\DeclareMathAlphabet{\mathpzc}{OT1}{pzc}{m}{it}

\newtheorem{propo}{Proposition}[section]
\newtheorem{lemma}[propo]{Lemma}

\newtheorem{coro}[propo]{Corollary}
\newtheorem{thm}[propo]{Theorem}
\newtheorem{theorem}[propo]{Theorem}

\newtheorem{remark}[propo]{Remark}

\def\hv{{\hat v}}

\def\cP{{\cal P}}

\def\reals{{\mathbb R}}
\def\prob{{\mathbb P}}
\def\E{\mathbb E}

\def\id{{\mathds I}}

\def\prb{{\pi}}
\def\hprb{{\widehat \pi}}
\def\Prb{{\Pi}}
\def\wt{{w}}
\def\hwt{{\hat w}}
\def\twt{{\tilde w}}
\def\Wt{{W}}
\def\hW{{\widehat W}}

\def\m{{m}}
\def\k{{r}}
\def\r{{r}}
\def\n{{n}}

\def\<{\langle}
\def\>{\rangle}
\def\M{M}
\def\hM{\widehat M}

\def\F{F}
\def\Fmiss{F_{\rm miss}}
\def\Fnoise{F_{\rm noise}}
\def\F{F}
\def\tF{\widetilde F}

\def\hG{\widehat G}
\def\tG{\widetilde G}
\def\hS{\widehat S}
\def\hU{\widehat U}
\def\hV{\widehat V}
\def\tV{R_3}
\def\hR{\widehat R}

\def\tS{\widetilde S}
\def\tU{\widetilde U}
\def\bU{\overline U}

\def\hpi{\hat \pi}
\def\hPi{\widehat \Pi}
\def\hPrb{\widehat \Pi}
\def\tpi{\tilde \pi}
\def\hw{\hat w}

\def\cP{{\cal P}}
\def\S{{\cal S}}
\def\cS{{\cal S}}
\def\hSigma{\widehat{\Sigma}}

\def\wmin{w_{\rm min}}
\def\wmax{w_{\rm max}}

\def\hv{\hat v}
\def\tv{\tilde v}
\def\hlambda{\widehat \lambda}

\def\Z{Z}
\def\hZ{\widehat Z}

\def\tmu{\tilde \mu}
\def\musp{\mu}

\def\hLambda{\widehat{\Lambda}}

\def\diag{{\rm diag}}
\def\poly{{\rm poly}}
\def\tepsilon{{\tilde{\varepsilon}}}
 
\def\Spectral{{\text{Algorithm 1}} } 
\def\matrixcompletion{{\text{\sc{MatrixAltMin}}}}

\def\tensorLS{{\text{\sc TensorLS}}}

\def\R{\reals}

\def\e{{\bf e}}

\def\Qw{\widehat{Q}}

\def\hA{\widehat{A}}
\def\hnu{\widehat{\nu}}
\def\hsm{\hSigma_{M_2}}
\def\hum{\hU_{M_2}}

\def\hQ{\widehat{Q}}
\def\tQ{\widetilde{Q}}
\def\Uc{\widetilde{Q}}
\def\hH{\widehat{H}}
\def\hX{\widehat{X}}
\newcommand{\ip}[2]{\langle #1,#2\rangle}

\begin{document}
\begin{titlepage}
\title{Learning Mixtures of Discrete Product Distributions\\ using Spectral Decompositions }

\author{Prateek Jain\\
{Microsoft Research India, Bangalore}\\
{prajain@microsoft.com}
\and Sewoong Oh \\
{Department of Industrial and Enterprise Systems Engineering}\\
{University of Illinois at Urbana-Champaign}\\
{swoh@illinois.edu}
}

\date{}

\maketitle\thispagestyle{empty}

\begin{abstract}

We study the problem of learning a distribution from samples, when 
the underlying distribution is a mixture of 
product distributions over discrete domains. 
This problem is motivated by several practical applications such as  
crowdsourcing, recommendation systems, and learning Boolean functions. 
The existing solutions either heavily rely on the fact that the number of mixtures is finite or 
have  sample/time complexity that is exponential in the number of mixtures.  
In this paper, we introduce a polynomial time/sample 
complexity  method for learning a mixture of $r$ discrete product distributions 
over $\{1, 2, \dots, \ell\}^n$, for  general $\ell$ and $r$. 
We show that our approach is consistent and further provide finite sample guarantees. 

We use recently developed techniques from 
tensor decompositions for moment matching. 
A crucial step 
in these approaches is 
to construct certain  tensors  
with low-rank spectral decompositions. % defined by the latent parameters. 
These tensors are typically estimated from the sample moments. %, which then is used to estimate the parameters.   
%some entries of these tensors cannot be directly estimated from samples. 
%Typically, such methods obtain the distribution parameters using fixed formulas involving intermediate variables that are eigenvalues/eigenvectors of the second or higher-order moments of the distribution. 
The main challenge in learning mixtures of discrete product distributions is that 
the corresponding low-rank tensors cannot be obtained directly from the sample moments. 
%However, the existing techniques do not apply directly to our problem, as the required intermediate variables cannot be obtained directly from the second or higher-order moments of the distribution. 
%s not clear if similar fixed formulas can be used directly. The reason being, some of the intermediate matrices/tensors cannot be obtained directly from sample estimates of the higher order moments of the distribution. 
Instead, we need to estimate a low-rank matrix using only off-diagonal entries, and 
estimate a tensor using a few linear measurements. 
We give an alternating minimization based method to estimate the low-rank matrix, and  
formulate the tensor estimation problem as a least-squares problem.
% that can decompose a given (nearly)low-rank + block-diagonal matrix into its low-rank and block-diagonal part. For (b), we reduce to problem to solving a system of linear equations. 

\end{abstract}
\end{titlepage}

%
%=========================================================================
%
\section{Introduction}
\label{sec:int}

%
%=========================================================================
%
Consider the following 
generative model 
for sampling from a mixture of product  distributions over discrete domains. 
We use $r$ to denote the number of components in the mixture, 
$\ell$ to denote the size of the discrete output alphabet in each coordinate, 
and $n$ to denote the total number of coordinates.  
Each sample belongs to one of $r$ components, 
and conditioned on its component $q\in\{1,\ldots,r\}$ 
the $n$ dimensional discrete sample $y\in\{1,\ldots,\ell\}^n$ 
is drawn from some distribution $\pi_{q}$. 
Precisely, the model is represented by the non-negative 
weights of the  components $w=[w_1 \ldots w_r]\in\reals^{r}$ that sum to one, 
and the $r$ distributions $\Pi = [\pi_1   \ldots \pi_r] \in \reals^{\ell\n\times \k}$. 
We use an $\ell\n$ dimensional binary random vector $x$ to represent a sample $y$. 
For $x = [x_1 \ldots x_n] \in\{0,1\}^{\ell \n}$,  the $i$-th coordinate 
$x_i\in\{0,1\}^\ell$  is an $\ell$ dimensional binary random vector such that 
\begin{eqnarray*}
 	x_{i}=e_j  \text{ if and only if } y_{i}=j \;, 
\end{eqnarray*}
where $e_j$ for some $j\in\{1,\ldots,\ell\}$ is the standard coordinate basis vector. 
%We use a matrix $\X$ to denote all the observed samples. 
%Define $\Prb \equiv [\pi^{(1)}, \ldots, \pi^{(n)}]^T \in \reals^{\ell\n\times \k} $, and let 
%$\prb_q\in\reals^{\ell\n}$ be the $q$-th column of $\Prb$ such that 
%$\E[x_i|t_i=a] = \prb_q^T$. 
%We assume that there is a vector $w\in\reals^\k$ of mixing weights $\wt=[\wt_1,\ldots,\wt_\k]$ 
%that sum to one 
%such that, when we sample a new task it has a type $q$ with probability $w_q$. 

When a sample is drawn, 
 the {\em type} of the sample is drawn from $w=[w_1 \ldots w_r]$ 
such that it has type $q$ with probability $w_q$. 
Conditioned on this type, the sample is distributed according to $\pi_q  \in \reals^{\ell\n}$, such that 
$y_i$'s are independent, hence it is a product distribution, and distributed according to  
\begin{eqnarray*}
	(\pi_q)_{(i,j)} = \prob( y_i=j \,|\, \text{$y$ belong to component $q$} )\;,
\end{eqnarray*}
where $(\pi_q)_{(i,j)}$ is the $\big((i-1)\ell+j \big)$-th entry of the vector $\pi^{(q)}$. 
Note that using the binary encoding, $\E[x|\text{its type is $q$}] = \pi_q$, 
and $\E[x] = \sum_q w_q \pi_q$.  
%and then generate a sample $y\in \R^{n}$ from this discrete distribution, where each $y_i \sim D_{iq}, 1\leq i\leq n$ and $D_{iq}$ is a discrete distribution over $[\ell]$; each $y_i$ is sampled independently. Alternatively, let $x_i=\e_j$, if $y_i=j$, where $\e_j$ is the $j$-th standard basis vector. 
%Also, we let $\pi_i\in \R^{\ell \times r}$ represent $D_{iq}$ where $\pi_i(j, q)=Prob(y_i=j|cluster=q)$. Then, the discrete distribution can be succinctly represented by the transition matrix $\Pi\in \R^{\ell n\times r}=[\pi_1; \pi_2; \dots; \pi_n]$ and the weights $W=[w_1, \dots, w_r]$. 
Also, we let $\pi^{(i)}\in \R^{\ell \times r}$ represent the distribution in the $i$-th coordinate 
such that  $\pi^{(i)}_{j, q}= (\pi_q)_{(i,j)}= \prob(y_i=j|\text{$y$ belongs to component  $q$})$. 
Then, the discrete distribution can be represented by the matrix 
$\Pi\in \R^{\ell n\times r}=[\pi^{(1)}; \pi^{(2)}; \dots; \pi^{(n)}]$ and the weights $w=[w_1, \dots, w_r]$. 

This mixture distribution (of $\ell$-wise discrete distributions over product spaces) captures as special cases 
the models used in several problems in  domains 
such as crowdsourcing \cite{DS79}, genetics \cite{SRH07}, and recommendation systems \cite{TM10}. 
For example, in the crowdsourcing application, this model is same as the popular Dawid and Skene \cite{DS79} model: 
$x_i$ represents answer of the $i$-th worker to a multiple choice question (or task) of type $q\in [r]$. 
Given the ground truth label $q$, each of the worker is assumed to answer independently. 
The goal is 
to find out the ``quality'' of the workers (i.e. learn $\Pi$) and/or
 to learn the type of each question (clustering). 

We are interested in the following two closely related problems:
\begin{itemize}
	\item Learn mixture parameters $\{\pi_q\}_{q \in \{1,\ldots,\k\}}$ and $\{\wt_q\}_{q\in\{1,\ldots,\k\}}$ accurately and efficiently. 
	\item Cluster the samples accurately and efficiently?
\end{itemize} 
Historically, however, different algorithms have been proposed depending on which question is addressed. 
Also, for each of the  problems, distinct measures of performances have been used to evaluate the proposed solution. In this paper, we propose an efficient method to address both questions. 

The first question of estimating 
the underlying parameters of the mixture components 
has been addressed in \cite{KMR94,FM99,FOS08}, where the error of a given algorithm is measured as the KL-divergence between the true distribution and the estimated distribution. More precisely, a mixture learning algorithm is said to be an {\em accurate learning algorithm}, if it outputs a mixture of product distribution 
such that the following holds with probability at least $1-\delta$: 
$$D_{\rm KL} \big(X \,||\,\hX \big)  \equiv \sum_x \prob(X=x) \log(\prob(X=x)/\prob(\hX=x))\leq \varepsilon ,\vspace*{-5pt}$$ where $\epsilon, \delta \in (0,1)$ are any given constants, and $X, \hX\in \{0, 1\}^{n\ell}$ denote the random vectors distributed according to the true and the estimated mixture distribution, respectively. Furthermore, the algorithm is said to {\em efficient} if its time complexity is polynomial in $\n$, $\k$, $\ell$, $1/\varepsilon$, and $\log(1/\delta)$. 

%let $x\in \{0,1\}^{\ell\n}$ denote the random vector 
%Under the full-rank and finite condition number assumption of Condition \ref{condition} ,  
%we give the first truly efficient learning algorithm 
%with running time $O()$. 
%Existing algorithms   
%either have strong assumptions on finite $\k$ and finite $\ell$, 
%or have a running time exponential in $\k$ and $\ell$. 
%Precisely, let $\Z\in \{0,1\}^{\ell\n}$ denote the random vector  
%distributed according to the true mixture distribution 
%that we draw samples from. 
%Now, a learning algorithm takes as input samples according to $\Z$, 
%a confidence parameter $\delta$,  
%and an approximation parameter $\varepsilon$ and 
%outputs an estimate of the latent variables: $\hprb_q$'s and $\hwt_q$'s for $q\%in\{1,\ldots,\k\}$. 
%Let $\hZ$ denote the random variable distributed according to the estimated parameters. 
%Following the definitions from \cite{KMR94,FM99}, 
%we say an algorithm is an {\em efficient learning algorithm}, 
%if it 
%outputs a mixture of product distributions 
%for any $\varepsilon\in(0,1)$ and any $\delta\in(0,1)$,  
%such that the Kullback-Leibler divergence, defined as $D_{\rm KL}(\Z\,||\hZ) \e%quiv \sum_z \prob(Z=z) \log(\prob(Z=z)/\prob(\hZ=z))$, 
% is at most $\varepsilon$: 
%\begin{eqnarray}%
%	D_{\rm KL} \big(\Z \,||\,\hZ \big) &\leq& \varepsilon \;, \label{eq:kl}
%\end{eqnarray}
%with probability at least $1-\delta$   
%in time polynomial in $\n$, $\k$, $\ell$, $1/\varepsilon$, and $\log(1/\delta). 

This Probably Approximately Correct (PAC) style framework was first introduced by Kearns et al. \cite{KMR94}, 
where they provided the first analytical result  
for a simpler problem of learning mixtures of Hamming balls, which is a special case of our model with $\ell=2$. However, the running time of the proposed algorithm is super-polynomial
 $O( (\n/\delta)^{\log \k} )$ and also assumes that one can obtain the exact probability of a sample $y$.  %Note that, there is no dependence on $\varepsilon$, as \cite{KMR94} assumes that one can submit a test input $y\in\{0,1\}^n$ 
%and get back the exact probability distribution of $y$ from the mixture distribution, which bypasses the finite sample error analysis. 
%  and 
% the use of KL-divergence as a standard measure of error have been first introduced by Kearns et al. \cite{KMR94}, 
% where they provided the first analytical result  
% for a simpler problem of learning mixtures of Hamming balls.  
% %which is a special case of mixture of product distributions over $\{0,1\}^n$. 
% Each component is a Hamming ball represented by a binary vector $x_i\in\{0,1\}^n$ for the $i$-th type.  
% A sample from the $i$-th component first takes $x_i$ 
% then flips each coordinate independently with a fixed probability $p$. 
% This is a special case of the model assumed in this paper, restricted to $\ell=2$ 
% and allowing only $\pi_q \in \{p,1-p\}^\n$ for all $q\in\{1,\ldots,\k\}$. 
% Even for this simple case the running time of the proposed algorithm is super-polynomial
%  $O( (\n/\delta)^{\log \k} )$.  
%  There is no dependence in $\varepsilon$, because 
%  it is assumed that one can submit a test input $y\in\{0,1\}^n$ 
% and get back the exact probability distribution of $y$ from the mixture distribution, 
% which bypasses the finite sample error analysis. 
Freund and Mansour \cite{FM99} were the first to 
addressed the sample complexity, but for the restrictive case of $\k=2$ and $\ell=2$. For this case, their method 
has running time $O( n^{3.5}  \log^3(1/\delta) / \varepsilon^5 )$  and sample complexity $O(n^2 \log(1/\delta) /\varepsilon^2)$. %Their approach   first reduces the search space to  three dimensions, one for the weight $\wt$ and two for each $\pi_1$ and $\pi_2$,  and greedily  searches over this  three-dimensional grid that cover the space. 
Feldman, O'Donnell, and Servedio in \cite{FOS08} 
generalized  approach of  \cite{FM99} to arbitrary number of types $\k$ 
and arbitrary number of output labels $\ell$. 
%This requires greedy search over $\k^3$-dimensional grid 
%with $\n/\varepsilon \ell^\ell $ points in each dimension. 
For general $\ell$, their algorithm requires  
running time scaling as $O((\n\ell^\ell/\varepsilon)^{\k^3})$. 
Hence, the proposed algorithm is 
an {\em efficient learning algorithm} 
%for all problem instances of $\{\pi_q\}_{q\in[\k]}$ and $\{w_q\}_{q\in[\k]}$, 
%but 
only for finite values of $\k=O(1)$ and $\ell=O(1)$. 

%In contrast, we give an {\em efficient learning algorithm} for all values of $\k$ and $\ell$, but for a set of problem instances with finite condition number as described in Condition \ref{condition}. 

A breakthrough in Feldman et al.'s result is that 
their result holds for all problem instances, with no dependence on the minimum weight $\wmin$ 
or the condition number $\sigma_1(\Pi\Wt^{1/2})/\sigma_\k(\Pi\Wt^{1/2})$, 
where $\sigma_i(\Pi\Wt^{1/2})$ is the $i$-th singular value of $\Pi\Wt^{1/2}$, 
and $\Wt$ is a $r\times r$ diagonal matrix with the weights $w$ in the diagonals. 
However, this comes at a cost of running time scaling exponentially in both $\k^3$ and $\ell$, which is unacceptable in practice for any value of $\k$ beyond two. Further, the running time is exponential for all problem instances, 
even when the problem parameters are {\em well-behaved}, 
with finite condition number. 

In this paper, we alleviate this issue by proposing an efficient algorithm for {\em well-behave} mixture distributions. In particular, we give an algorithm with polynomial running time, 
and prove that it gives $\varepsilon$-accurate estimate 
for any problem instance that satisfy the following two conditions: 
$a$) the weight $w_q$ is strictly positive for all $q$; and  $b$) the condition number $\sigma_1(\Pi\Wt^{1/2})/\sigma_\k(\Pi\Wt^{1/2})$ is bounded as per hypotheses in Theorem \ref{thm:finite}. 
% SEWOONG: I changed finite to bounded, because it is sufficient for it to be polynomial as long as it satisfies the conditions in Therme 3.3 

The existence of an efficient learning algorithm for all problem instances  and parameters still remains 
an open problem, and 
it is conjectured in \cite{FOS08} that 
``solving the mixture learning problem for any $\k=\omega(1)$ would require a major breakthrough in learning theory''. 

\begin{table}[h]
\begin{center}
\begin{tabular}{| c | c | c |}
\hline
& $\k,\ell=O(1)$& General $\k$ and $\ell$\\ \hline
 $\sigma_1(\Pi\Wt^{1/2})/\sigma_\k(\Pi\Wt^{1/2}) = poly(\ell, r, n) $ & {\sc WAM}\cite{FOS08}, \Spectral & \Spectral \\\hline
 General cond. number & {\sc WAM} \cite{FOS08} & Open  \\ \hline
\end{tabular}
\caption{Landscape of efficient learning algorithms}
\end{center}\vspace*{-10pt}
\end{table}

The second question finding the clusters 
has been addressed in \cite{CHR07,CR08}. 
Chaudhuri et al. in \cite{CHR07} introduced an iterative clustering algorithm 
but their method is restricted to the case of a mixture of two product distributions with binary outputs, i.e. $\k=2$ and $\ell=2$. %However, the algorithm as well as the analysis heavily relies on the fact that there are only two clusters, and cannot be generalized.  
Chaudhuri and Rao in \cite{CR08} proposed a spectral method for general $r, \ell$. However, for the algorithm to correctly recover cluster of each sample w.h.p, the underlying mixture distribution should 
%of projecting the samples onto a $\k$-dimensional  subspace spanned by 
%the top singular vectors of the empirical covariance matrix. 
%To ensure that the algorithm finds a good subspace that separates the centers of each type, 
%it is required that the underlying mixture distribution 
satisfy a certain `spreading' condition. %When there is sufficient spread in the underlying distribution, they show that the types of each sample can be recovered with high probability. 
Moreover, the algorithm need to know  
the parameters characterizing the `spread' of the distribution, 
which typically is not available apriori.  
Although it is possible to estimate the mixture distribution, once the samples are clustered, 
Chaudhuri et al. provides no guarantees for estimating the distribution. 
%no guarantee is known for estimating the distributions. 
%Finally, their spectral approach is not consistent, in the sense that even if infinite samples are available, the algorithm does not guarantee deterministic recovery of the true cluster for each point. \#\#\#@Prateek: Pleas check this line. 
As is the case for the first problem, for clustering also, 
we provide an efficient algorithm for general $\ell, r$, 
under the assumption that the condition number of $\Pi\Wt^{1/2}$ to be bounded. 
% SEWOONG: I changed finite to bounded, because it is sufficient for it to be polynomial as long as it satisfies the conditions in Therme 3.3 
This condition is not directly comparable with the spreading condition assumed in previous work. 
Our algorithm first estimates the mixture parameters and then uses the distance based clustering method of \cite{AK01}. %Moreover, our algorithm doesn't require  spread assumption. % and is consistent as well (as long as $n\geq \Omega(r^{4})$). 

Our method for estimating the mixture parameters is based on the moment matching technique from \cite{AHK12, AroraGMS12}.  
 Typically, second and third (and sometimes fourth) moments of the true distribution are estimated using the given samples. Then, using the spectral decomposition of the second moment one develops certain whitening operators that reduce the higher-order moment tensors to orthogonal tensors. Such higher order tensors are then decomposed using a power-method based method \cite{AGHKT12} to obtain the required distribution parameters. 

While such a technique is generic and applies to several popular models \cite{HK13, AGHKT12}, for many of the models the moments themselves constitute the ``correct'' intermediate quantity that can be used for whitening and tensor decomposition. However, because there are dependencies in the $\ell$-wise model (for example, $x_1$ to $x_\ell$ are correlated), the higher-order moments are ``incomplete'' versions of the intermediate quantities that we require (see \eqref{eq:def_m2}, \eqref{eq:def_m3}). Hence, we need to complete these moments so as to use them for estimating distribution parameters $\Pi, W$. 

Completion of the ``incomplete'' second moment, can be posed as a 
low-rank matrix completion problem where the {\em block-diagonal} elements are missing. For this problem, we propose an alternating minimization based method and, borrowing techniques from the recent work of \cite{JNS13}, we prove that alternating minimization is able to complete the second moment exactly. We would  like to note that our  alternating minimization result also solves a generalization of the low-rank+diagonal decomposition problem of \cite{SaundersonCPW12}. Moreover, unlike trace-norm based method of \cite{SaundersonCPW12}, which in practice is computationally expensive, our method is efficient, requires only one 
Singular Value Decomposition (SVD) step, and is robust to noise as well.%, which is critical as there is bound to be certain noise in second moment estimation using the given samples. 

We reduce the completion of the ``incomplete'' third moment to a simple least squares problem that is robust as well. Using techniques from our second moment completion method, we can analyze an alternating minimization method also for the third moment case as well. However, for the mixture problem we can exploit the structure to reduce the problem to an efficient least squares problem with closed form solution. 

Next, we present our method (see Algorithm~\ref{algo:main}) that combines the estimates from the above  mentioned steps to estimate the distribution parameters $\Pi, W$ (see Theorem~\ref{thm:consistency}, Theorem~\ref{thm:finite}). After estimating the model parameters $\Pi$, and $W$, we also show that the KL-divergence measure 
%(popularly used for the problem (a) ) 
and the clustering error measure 
%(used for the problem (b)) 
can also be shown to be small.
In fact the excess error vanishes as the number of samples grow (see Corollary~\ref{coro:KL}, Corollary~\ref{coro:cluster}). 
%and hence avoid the alternating minimization method. %this problemthat can be analyzed similar using techniques for our second moment based method, but for the mixture distribution problem, we 
% Even if we are given the perfect covariance matrix, 
% the problem of extracting the good subspace is not easy. 

\section{Related Work}

Learning mixtures of distributions is an important problem with several applications such as clustering, crowdsourcing, community detection etc. One of the most well studied problems in this domain is that of learning a mixture of Gaussians. 
There is a long list of  interesting recent results, and discussing the literature in detail is out side of the scope of this paper. 
Our approach is inspired by both spectral and moment-matching based techniques that have been successfully applied 
in learning a mixture of Gaussians \cite{VempalaW04, AK01, MoitraV10, HK13}. 

Another popular mixture distribution  arises in topic models, where each word $x_i$ is selected from a $\ell$-sized dictionary. Several recent results show that such a model can also be learned efficiently using spectral as well as moments based methods \cite{RabaniSS12, AHK12, AroraGM12}. However, there is a crucial difference between the general mixture of product distribution that we consider and the topic model  distribution. Given a topic (or question) $q$, each of the words $x_i$ in the topic model have exactly the same probability. That is, $\pi^{(i)}=\pi$ for all $i\in\{1,\ldots,n\}$. 
In contrast, for our problem,  $\pi^{(i)}\neq \pi^{(j)}, i\neq j$, in general. %x_i$ can have a completely different distribution than $x_j$ where $i\neq j$.

Learning mixtures of discrete distribution over product spaces has 
several practical applications such as  
crowdsourcing, recommendation systems, etc. However, as 
discussed in the previous section, most of the existing results for this 
problem are designed for the case of small alphabet size $\ell$ or 
the number of mixture components $r$. 
For several practical problems \cite{KOS13SIGMETRICS}, $\ell$ can be  large and 
hence existing methods either do not apply or are very inefficient. In this work, we 
propose first provably efficient method for learning mixture of discrete distributions 
for general $\ell$ and $r$. 

%A related class of problems arises in the topic modelling problem,  

%Another Gaussian distribution has been one of the most widely studied distribution for this problem \cite{Das99}. %Starting from Dasgupta's celebrated paper \cite{Das99}, past decade has seen tremendous progress in learning mixtures of Gaussians. 
%Typical results for this problem show that, if the separation between the means is much larger than the variance of each Gaussian, then  the mean and the variance of each Gaussian can be estimated exactly \cite{VempalaW04, AK01, MoitraV10}. One exception  has been the work of \cite{HK13}, that doesn't require any separation assumption but the sample complexity increases as the gap between means get smaller. 

%Another popular class of distribution is the class of discrete distributions over $\ell$-dimensional product spaces, i.e., the discrete distribution is over $\ell$-sized space. Such distributions are  relevant for several practical problems in the area of crowd-sourcing, recommendation systems, learning Boolean functions etc. 
%Moreover, there exists several powerful results for this problem. 

%Another related product distribution is the standard topic model, 
Our method is based on tensor decomposition methods for moment matching 
that have recently been made popular for learning mixture distributions.   
%The key idea behind such methods is that for several practical distributions,  a small number of higher-order moments serve as a ``signature''. That is, by using a few higher order moments, one can find out the parameters of the mixture. However, most of these methods require spectral decomposition of the higher order moment tensor which in general is NP-hard. To alleviate this issue, the existing methods first {\em orthogonalize} the tensor using certain {\em whitening} operators. Interestingly, {\em orthogonal} tensors are known to be amenable to decomposition using a simple power-method based algorithm \cite{AGHKT12}. Owing to these developments, there have been several interesting results in this area recently.  
For example, \cite{HK13} provided a method  
to learn mixture of Gaussians without any separation assumption. Similarly, \cite{AHK12} introduced a method for learning mixture of HMMs, and also for topic models. Using similar techniques, another interesting result has been obtained for the problem of independent component analysis (ICA) \cite{AroraGMS12, GoyalL12, HK13}. % See Most of these methods Recently, \cite{AGHKT12} generalized these methods and studied Learning mixture discrete distribution with different structures: 

Typically, tensor decomposition methods proceed in two steps. 
First, obtain a whitening operator using the second moment estimates. 
Then, use this whitening operator to construct a tensor with orthogonal decomposition, 
which reveals the true parameters of the distribution. 
However, in a mixture of $\ell$-way distribution that we consider, 
the second or the third moment do not reveal all the ``required'' 
entries, making it difficult to find the standard whitening operator. 
%do not succeed for our problem. 
We handle this problem by posing it as a matrix completion problem and using 
an alternating minimization method to complete the second moment. 
%We show that one can think of the problem as that of low-rank matrix completion with missing block-diagonal entries. For this problem, we give an alternating minimization based method and prove guarantees for the same. 
Our proof  for the alternating minimization method closely follows the analysis of \cite{JNS13}. However, \cite{JNS13} handled a matrix completion problem where the entries are missing uniformly at random, while in our case the block diagonal elements are missing. % that introduced a method to analyze alternating minimization for matrix completion but where entries are missing uniformly at random, instead of correlated omissions that we observe. 
%Learning mixture of product distributions with different distribution: Gaussian.  
%Starting from Dasgupta's celebrated paper \cite{Das99}, past decade has seen 
%tremendous progress in learning mixtures of distributions. 

\subsection{Notation} 
\label{sec:notation}
Typically, we denote a matrix or a tensor by an upper-case letter (e.g. $M$) while a vector is denoted by a small-case letter (e.g. $v$). $M_i$ denotes the $i$-th column of matrix $M$. $M_{ij}$ denotes the $(i,j)$-th entry of matrix $M$ and $M_{ijk}$ denotes the $(i,j,k)$-th entry of the third order {\em tensor} $M$. $A^T$ denotes the transpose of matrix $A$, i.e., $A^T_{ij}=A_{ji}$. $[k]=\{1,\ldots,k\}$ denotes the set of first $k$ integers. $\e_i$ denotes the $i$-th standard basis vector. 

If $M\in \R^{\ell n \times d}$, then $M^{(m)}$ ($1\leq m\leq n$) denotes the $m$-th block of $M$, i.e., $(m-1)\ell+1$ to $m\ell$-th rows of $M$. 
The operator $\otimes$ denotes the outer product. For example, $H=v_1\otimes v_2 \otimes v_3$ denote a rank-one tensor such that $H_{abc}=(v_1)_a\cdot (v_2)_b\cdot (v_3)_c$. 
For a symmetric  third-order tensor $T\in\reals^{d\times d\times d}$, define 
%a $d\times d$ dimensional projection   
%$T[\eta] \equiv \sum_{i_1\in[d]} \sum_{i_2, i_3\in[d]} T_{i_1,i_2,i_3} \eta_{i_3} (e_{i_1}\otimes e_{i_2})$, and define 
an  $r\times r\times r$ dimensional operation with respect to a matrix $R\in\reals^{d\times r}$  as \vspace*{-5pt}
$$T[R,R,R] \equiv  \sum_{i_1,i_2,i_3\in[d]} T_{i_1,i_2,i_3} R_{i_1,j_1} R_{i_2,j_2} R_{i_3,j_3} (e_{j_1}\otimes e_{j_2}\otimes e_{j_3}).\vspace*{-5pt}$$
$\|A\|=\|A\|_2$ denotes the spectral norm of a tensor $A$. 
That is, $\|A\|_2=\max_{x, \|x\|=1}A[x, \dots, x]$. $\|A\|_F$ denotes the Frobenius norm of $A$, i.e., $\|A\|_F=\sqrt{\sum_{i_1, i_2, \dots, i_p}A_{i_1i_2\dots i_p}^2}$. 
We use $M=U\Sigma V^T$ to denote the singular value decomposition (SVD) of $M$, 
where $\sigma_r(M)$ denotes the $r$-th singular value of $M$. Also, wlog, assume that $\sigma_1\geq \sigma_2\dots \geq \sigma_r$. 
%
%=========================================================================
%
\section{Main results}
\label{sec:result}
In this section, we present our main results for estimating the mixture weights $w_q, 1\leq q\leq r$ and the probability matrix $\Pi$ of the mixture distribution. Our estimation method is based on the moment-matching technique that has been popularized by several recent results \cite{AHK12,HKZ12a,HK13,AGHKT12}. However, our method differs from the existing methods in the following crucial aspects: we propose 
%we will highlight the differences at appropriate places in this section. 
$(a)$ a matrix completion approach to estimate the second moments from samples (Algorithm~\ref{algo:altmin}); and 
$(b)$ a least squares approach with an appropriate change of basis to estimate the third moments from samples (Algorithm~\ref{algo:ls}).
These approaches provide robust algorithms to estimating the moments and might be of independent interest to 
 a broad range of applications in the domain of learning mixture distributions. 

The key step in our method is estimation of the following two quantities: 
\begin{eqnarray}
%	\M_1 &\equiv& \sum_{q\in[k]} \wt_\a\, \prb_\a = \Prb \wt \;,\\
	\M_2 &\equiv& \sum_{q\in[\k]} \wt_q\, \big(\prb_q\otimes\prb_q\big) = \Prb \Wt\Prb^T \; \in\; \reals^{\ell\n\times\ell\n} \;,\label{eq:def_m2}\vspace*{-5pt}\\ 
	\M_3 &\equiv& \sum_{q\in[\k]} w_q\, \big(\prb_q\otimes\prb_q\otimes\prb_q\big) \;\in\;\reals^{\ell\n\times\ell\n\times\ell\n} \;,\label{eq:def_m3}\vspace*{-5pt}
\end{eqnarray}
where $W$ is a diagonal matrix s.t. $W_{qq}=w_q$. 
%, and $\Pi=U\Sigma V^T$ denotes SVD of $\Pi$. 

Now, as is standard in the moment based methods, we exploit spectral structure of $M_2, M_3$ to recover the latent parameters $\Pi$ and $W$. The following theorem presents a method for estimating $\Pi, W$, assuming $M_2, M_3$ are estimated {\em exactly}: 

\begin{theorem}
  Let $M_2, M_3$ be as defined in \eqref{eq:def_m2}, \eqref{eq:def_m3}. Also, let $M_2=U_{M_2}\Sigma_{M_2} U_{M_2}^T$ be the eigenvalue decomposition of $M_2$. Now, define $G=M_3[U_{M_2}\Sigma_{M_2}^{-1/2}, U_{M_2}\Sigma_{M_2}^{-1/2}, U_{M_2}\Sigma_{M_2}^{-1/2}]$. Let $V^G=[v^G_1\, v^G_2\, \dots v^G_r]\in \R^{r\times r}$, $\lambda^G_q, 1\leq q\leq r$ be the eigenvectors and eigenvalues obtained by the orthogonal tensor decomposition of $G$ (see \cite{AGHKT12}), i.e., $G=\sum_{q=1}^r \lambda^G_q (v^G_q\otimes v^G_q \otimes v^G_q)$. Then, 
$$\Pi=U_{M_2} \Sigma_{M_2}^{1/2} \,V^G\,  \Lambda^G\;,\;\;\text{ and }\;\;\;\;\quad W=(\Lambda^G)^{-2},$$
where $\Lambda^{G}\in \R^{r\times r}$ is a diagonal matrix with $\Lambda^G_{qq}=\lambda^G_q$. \label{thm:main_comps}
\end{theorem}
\noindent The above theorem reduces the problem of estimation of mixture parameters $\Pi, W$ to that of estimating $M_2$ and  $M_3$. 
Typically, in moment based methods, tensors  corresponding to $M_2$ and $M_3$ 
can be estimated directly using the second moment or third moment of the distribution, which can be estimated 
efficiently using the provided data samples. 
In our problem, however, the block-diagonal entries of $\M_2$ and $\M_3$ 
cannot be directly computed from these sample moments. 
For example, the expected value of a diagonal entry at $j$-th coordinate is 
$\E[x  x^T]_{j,j} = \E[x_{j}] = \sum_{q\in[\k]} \wt_q \Pi_{j,q}$, 
where as the corresponding entry for $\M_2$ is $(\M_2)_{j,j} = \sum_{q\in[\k]} \wt_q (\Pi_{j,q})^2$. 
%The sample second moments do not give consistent estimates for the block-diagonals as shown in the figure below.

To recover these unknown $\ell\times\ell$ block-diagonal entries of $\M_2$, we use an alternating minimization algorithm. Our algorithm writes $M_2$ in a bi-linear form and solves for each factor of the bi-linear form using the computed off-diagonal blocks of $M_2$. We then prove that this algorithm exactly recovers the  missing entries when we are given the exact second moment. For estimating $M_3$,  we reduce the problem of estimating unknown block-diagonal entries of $\M_3$ to a least squares problem that can be solved efficiently.

Concretely, to get a consistent estimate of $\M_2$, we pose it as a matrix completion problem, where 
we use the off-block-diagonal entries of the second moment, 
which we know are consistent, to estimate the missing entries. 
Precisely, let \vspace*{-5pt}
 $$\Omega_2 \;\equiv\;  \Big\{\,(i,j) \subseteq [\ell\n]\times [\ell\n] \,|\, \lceil \frac{i}{\ell} \rceil \neq\lceil \frac{j}{\ell}\rceil  \,\Big\}, \vspace*{-5pt}$$
 be the indices of the off-block-diagonal entries, and define 
a masking operator as: 
\begin{eqnarray}
	\label{eq:defcP}
	 	\cP_{\Omega_2}(A)_{i,j} &\equiv& \left\{ \begin{array}{rl} A_{i,j}\;, & \text{ if } (i,j)\in\Omega_2  \;,\\ 0\;, & \text{ otherwise . } \end{array}\right. \vspace*{-5pt}
\end{eqnarray}
Now, using the fact that $\M_2$ has rank at most $r$, we find a rank-$\r$ estimate that explains the off-block-diagonal entries  
using an alternating minimization algorithm defined in Section \ref{sec:algorithm}. \vspace*{-5pt}
\begin{eqnarray}
	\hM_2 &\equiv& \matrixcompletion\left(\frac{2}{|\cS|} \sum_{t\in [|\cS|/2]} x_t x_t^T ,\Omega_2, r, T\right)\;, \label{eq:defcp2}\vspace*{-5pt}
\end{eqnarray}
where $\{x_1, \dots, x_{|\cS|}\}$ is the set of observed samples, and $T$ is the number of iterations. 
We use the first half of the samples to estimate $M_2$ and the rest to estimate the third-order tensor. 

Similarly for the tensor $\M_3$, the sample third moment does not converge to $\M_3$. However, the off-block diagonal  entries do converge to the corresponding entries of $\M_3$. That is, let 
 $$\Omega_3 \;\equiv\; \Big\{\,(i,j,k) \subseteq [\ell\n]\times [\ell\n]\times[\ell\n] \,|\, \lceil \frac{i}{\ell} \rceil \neq \lceil \frac{j}{\ell} \rceil \neq \lceil \frac{k}{\ell} \rceil \neq \lceil \frac{i}{\ell} \rceil \,\Big\},\vspace*{-5pt}$$
 be the indices of the off-block-diagonal entries, and define the following masking operator: 
\begin{eqnarray}
	\label{eq:defcP3}
	 	\cP_{\Omega_3}(A)_{i,j,k} &\equiv& \left\{ \begin{array}{rl} A_{i,j,k}\;, & \text{ if } (i,j,k)\in\Omega_3 \;,\\ 0\;, & \text{ otherwise . } \end{array}\right. \vspace*{-5pt}
\end{eqnarray}
Then, we have consistent estimates for $\cP_{\Omega_3}(\M_3)$ from the sample third moment. 

Now, in the case of $M_3$, we do not explicitly compute $\M_3$. 
Instead, we estimate a $\k\times\k\times\k$ dimensional tensor $\tG \equiv M_3[\hU_{M_2}\hSigma_{M_2}^{-1/2}, \hU_{M_2}\hSigma_{M_2}^{-1/2}, \hU_{M_2}\hSigma_{M_2}^{-1/2}]$ 
(cf. Theorem~\ref{thm:main_comps}), 
using a least squares formulation that uses only off-diagonal blocks of $P_\Omega(M_3)$. That is, \vspace*{-5pt}%we approximate $\hG$ using: 
%Similarly for the tensor $\M_3$, we want to get a consistent estimate $\hV=[\hv_1 \ldots \hv_\k]$ 
%such that $\M_3[US^{-1/2},US^{-1/2},US^{-1/2}] = \sum_{a\in[\k]}   \wt^{-1/2}  (\hv_a\otimes \hv_a \otimes \hv_a)$.  
%The sample third moments, however,  do not converge to $\M_3$ 
%in the block-diagonal entries. 
% We use these estimates to find an orthogonal matrix $\hV$ 
% and the weights $\hwt$ that explains the off-block-diagonal entries of the third moment.  
%To this end, we introduce an alternating minimization algorithm for tensor completion in Section \ref{sec:algorithm}. 
%The algorithm outputs an estimate of the tensor with orthogonal decomposition 
%$\hM=\sum_{a\in[k]} \hwt_a^{-1/2} (\hv_a\otimes\hv_a\otimes\hv_a)$:
\begin{eqnarray*}
	\hG &\equiv& \tensorLS\Big(\frac{2}{|\cS|} \sum_{t=1+|\cS|/2}^{|\cS|} x_t\otimes x_t\otimes x_t ,\Omega_3,\hU_{M_2},\hSigma_{M_2}  \Big) \;,\vspace*{-5pt}
\end{eqnarray*}
where $\hM_2=\hU_{M_2}\hS_{M_2}\hU_{M_2}^T$ is the singular value decomposition of 
the rank-$r$ matrix $\hM_2$. 
%$\hU\hS^{-1/2}$ is the {\em whitening operator} \tensorcompletion{} usesto orthogonalize the $\M_3$. 
After estimation of $\hG$, similar to Theorem~\ref{thm:main_comps}, we use the whitening and tensor decomposition to estimate $\Pi, W$. 
See Algorithm~\ref{algo:main} for a pseudo-code of our approach. 

{\bf Remark}: Note that we use a new set of $|\cS|/2$ samples to estimate the third moment. 
This sub-sampling helps us in our analysis, as it ensures independence of the samples 
$x_{|\cS|/2+1}, \dots, x_{|\cS|}$ from the output of the alternating minimization step \eqref{eq:defcp2}. 

\begin{algorithm}[t!]
\caption{{\bf Spectral-Dist:} Moment method for Mixture of Discrete Distribution}
\label{algo:main}
\begin{algorithmic}[1]
  \STATE Input: Samples $\{x_t\}_{t\in\cS}$
\STATE $\hM_2\leftarrow \matrixcompletion\left( \left(\frac{2}{|\cS|} \sum_{t\in [|\cS|/2]} x_t x_t^T\right) ,\Omega_2, r, T \right)$ \hfill(see Algorithm~\ref{algo:altmin})
\STATE Compute eigenvalue decomposition of $\hM_2=\hU_{M_2} \hSigma_{M_2} \hU_{M_2}^T$
\STATE $\hG\leftarrow \tensorLS\left( \left(\frac{2}{|\cS|}\sum_{t=|\cS|/2+1}^{|\cS| }x_t\otimes x_t\otimes x_t\right), \Omega_3, \hU_{M_2}, \hsm\right)$ \hfill (see Algorithm~\ref{algo:ls})
%\STATE $\hL\leftarrow \hG[\hSigma_{M_2}^{-1/2}, \hSigma_{M_2}^{-1/2}, \hSigma_{M_2}^{-1/2}]$ 
\STATE Compute a rank-$r$ orthogonal tensor decomposition $\sum_{q\in [r]} \hlambda_{q}^G (\hv^G_q\otimes \hv^G_q \otimes \hv^G_q)$ of $\hG$, using Robust Power-method of \cite{AGHKT12}
\STATE Output: $\hPi=\hU_{M_2} \hSigma_{M_2}^{1/2} \hV^G \hLambda^G$, $\hW=(\hLambda_q^G)^{-2}$, where $(\hV^G)^T=[\hv^G_1\ \dots \ \hv^G_r]$ 
\end{algorithmic}
\end{algorithm}

The next theorem shows that the moment matching approach (Algorithm~\ref{algo:main}) is consistent. 
Let $\hW=\diag([\hwt_1, \dots, \hwt_r])$ and $\hPi=[\hprb_1, \dots, \hprb_r]$ denote the estimates 
obtained using Algorithm~\ref{algo:main}. 
Also, let $\mu$ denote the block-incoherence 
of $M_2=\Pi W \Pi^T$ 
as defined in \eqref{eq:defincoherence}. 

\begin{theorem}
	\label{thm:consistency}
	Assume that the sample second and the third moments are exact, i.e., \\
	$\cP_{\Omega_2}(\frac{2}{|\cS|}\sum_{t\in[|\cS|/2]} x_tx_t^T) =\cP_{\Omega_2}(\M_2)$ 
	and $\cP_{\Omega_3}(\frac{2}{|\cS|}\sum_{t=|\cS|/2+1}^{|\cS|} x_t\otimes x_t\otimes x_t) = \cP_{\Omega_3}(\M_3)$. Also,  let $T=\infty$ for the  \matrixcompletion procedure and let $\n\,\geq \,{C\,\sigma_1(M_2)^5\mu^5 r^{3.5}}/{\sigma_r(M_2)^5}$, for a global constant $C>0$. Then, there exists a permutation $P$ over $[r]$ such that, for all $q\in[r]$, \vspace*{-5pt}
	\begin{eqnarray*}
		\pi_q = \widehat{\pi}_{P(q)}	\;\;\text{ and }\;\;\;\; w_q = \hwt_{P(q)} \;. \vspace*{-5pt}
	\end{eqnarray*}
\end{theorem}

We now provide a finite sample version of the above theorem. %bound for Algorithm~\ref{algo:main}. 

\begin{theorem}[Finite sample bound] 
	\label{thm:finite}
	There exists 
	positive  constants $C_0$, $C_1$, $C_2$, $C_3$ and 
	a permutation $P$ on $[r]$ such that
	if $n\geq {C_0 \,\sigma_1(M_2)^{4.5}\mu^4 r^{3.5}}/{\sigma_r(M_2)^{4.5}}$ then  
	for any $\varepsilon_M\leq \frac{C_1}{\sqrt{r+\ell}}$ and for a large enough sample size: \vspace*{-5pt}
	\begin{eqnarray*}
%		|\cS| &\geq& C_2\frac{n^2\log(2\n\ell/\delta)\|M_2\|_F^2r}{\sigma_r(M_2)^4\cdot \varepsilon_M^2}\;,\vspace*{-5pt}
		|\cS| &\geq& C_2\,\frac{\mu^6 \,r^6}{\wmin}\,\frac{ \sigma_1(M_2)^6 n^3 }{\sigma_r(M_2)^9 } \frac{\log(\n/\delta)}{\varepsilon_M^2}\;,\vspace*{-5pt}
	\end{eqnarray*}
	the following holds for all $q\in [r]$, with probability at least $1-\delta$:\vspace*{-5pt} %	for all  $q\in [r]$, 
	\begin{eqnarray*}
		|\hwt_{P(q)} - \wt_q |&\leq& \varepsilon_M \;, \text{  }\vspace*{-5pt}\\
		\|\hprb_{P(q)} - \prb_q\| &\leq& \varepsilon_M \sqrt{\frac{r\,\wmax\, \sigma_1(M_2)}{w_{min}}}   \;.\vspace*{-5pt}
	\end{eqnarray*} 
	Further,  \Spectral runs in 
	time $\poly\big(\, \n,\ell,\k,1/\varepsilon, \log(1/\delta), 1/\wmin,\sigma_1(\M_2)/\sigma_r(\M_2) \,\big)$. 
\end{theorem}
\noindent Note that, the estimated $\hpi_i$'s and $\hwt_i$'s using Algorithm~\ref{algo:main} do not necessarily define a {\em valid} probability measure: 
they can take negative values and might not sum to one. 
We can process the estimates further to get a valid probability distribution, 
and show that the estimated mixture distribution is close in Kullback-Leibler divergence to the original one. 
Let $\varepsilon_w = C_3 \varepsilon_M/\sqrt{\wmin}$. 
We first set \vspace*{-5pt}
\begin{eqnarray*}
	\twt'_q &=& \left\{ \begin{array}{rl} 
		\hwt_q & \text{ if } \hwt_q \geq  \varepsilon_w \;,\\
		\varepsilon_w & \text{ if }  \hwt_q< \varepsilon_w\;, 
	\end{array}
	\right.\vspace*{-5pt}
\end{eqnarray*}
and set mixture weights $\twt_q=\twt'_q/\sum_{q'} \twt'_{q'}$. 
Similarly, let $\varepsilon_\pi=C_3 \varepsilon_M\sqrt{\frac{\sigma_1(M_2)\,r(1+\varepsilon_M\sigma_r(M_2))}{w_{min}}}$ and  set \vspace*{-5pt}
\begin{eqnarray*}
	\tpi'^{(j)}_{q,p} &=& \left\{ \begin{array}{rl} 
		\hpi^{(j)}_{q,p} & \text{ if } \hpi^{(j)}_{q,p}\geq  \varepsilon_\pi \;,\\
		\varepsilon_\pi & \text{ if }  \hpi^{(j)}_{q,p} < \varepsilon_\pi\;, 
	\end{array}
	\right.\vspace*{-5pt}
\end{eqnarray*}
for all $q\in[\k]$, $p\in[\ell]$, and  $j\in[\n]$, 
and normalize it to get valid distributions $\tpi^{(j)}_{q,p}=\tpi'^{(j)}_{q,p}/\sum_{p'}\tpi'^{(j)}_{q,p'}$. 
Let $\hX$ denote a random vector in $\{0,1\}^{\ell\n}$ obtained by 
first selecting a random type $q$ with probability $\twt_q$ 
and then drawing from a random vector according to $\tpi_q$.
\begin{coro}[KL-divergence bound]
	\label{coro:KL}
	Under the hypotheses of Theorem \ref{thm:finite}, 
	there exists a positive constant $C$ such that 
	if $|\cS|\geq C n^7 r^7  \mu^6 \sigma_1(M_2)^7 \ell^{12} \wmax \log( n /\delta)  /(\sigma_r(M_2)^9 \eta^6 \wmin^2)$, then   
	\Spectral with the above post-processing produces 
	a $\k$-mixture distribution $\hX$ 
	that, with probability at least $1-\delta$, satisfies : $D_{KL}(X||\hX) \leq \eta $. 
%	\begin{eqnarray*}%
%		D_{KL}(\Z||\hZ) &\leq& \eta \;.
%	\end{eqnarray*}
\end{coro}

Moreover, we can show that the ``type'' of each data point can also be recovered accurately. 	
\begin{coro}[Clustering bound]
	\label{coro:cluster} 
	Define: \vspace*{-5pt}
	\begin{eqnarray*}
		 \tepsilon &\equiv & \max_{i,j\in[\k]} \,\left\{ \frac{\|\pi_i-\pi_j\|^2 -  2\|\Pi\|_F \sqrt{2 \log(\k/\delta)}}{(\|\pi_i-\pi_j\| + 2 \sqrt{2\log(\k/\delta)})\k^{1/2} } \right\} \;. \vspace*{-5pt}
	\end{eqnarray*}
	Under the hypotheses of Theorem \ref{thm:finite}, there exists a positive numerical constant $C$ such that 
	if $\tepsilon>0$ and $|\cS| \geq C \mu^6 r^7 n^3 \sigma_1(M_2)^7 \wmax \log(n/\delta)/ (\wmin^2 \sigma_r(M_2)^9 \tepsilon^2)$, then  
	with probability at least $1-\delta$, 
	the distance based clustering algorithm of \cite{AK01} 	
	computes a correct clustering of the samples. 
\end{coro}

%%% Local Variables: 
%%% mode: latex
%%% TeX-master: "crowd"
%%% End: 

%
%=========================================================================
%
\section{Algorithm}
\label{sec:algorithm}

\begin{algorithm}[t!]
\caption{\matrixcompletion: Alternating Minimization for Matrix Completion}
\label{algo:altmin}
\begin{algorithmic}[1]
\STATE Input: $S_2=\frac{2}{|\S|}\sum_{t\in \{1,\ldots, |\cS|/2\} }x_tx_t^T$, $\Omega_2$, $r$, $T$
\STATE Initialize $\ell\n\times\r$ dimensional matrix $U_0\leftarrow $ top-$r$ eigenvectors of $\cP_{\Omega_2}(S_2)$
\FORALL{$\tau=1 $ to $T-1$}
\STATE $\hU_{\tau+1}=\arg\min_{U}\|\cP_{\Omega_2}(S_2)-\cP_{\Omega_2}(UU_\tau^T)\|_F^2$
\STATE $[U_{\tau+1} R_{\tau+1}]={\rm QR}(\hU_{t+1})$ \hspace{6cm} \hfill (standard QR decomposition)
\ENDFOR
\STATE Output: $\hM_2=(\hU_{T})(U_{T-1})^T$
\end{algorithmic}
\end{algorithm}

\begin{algorithm}[t!]
\caption{\tensorLS: Least Squares method for Tensor Estimation}
\label{algo:ls}
\begin{algorithmic}[1]
\STATE Input: $S_3=\frac{2}{|\S|}\sum_{t\in \{|\cS|/2+1,\ldots,|\cS|\}}(x_t\otimes x_t\otimes x_t) $, $\Omega_3$, $\hU_{M_2}$, $\hsm$
\STATE Define operator $\hnu: \R^{r\times r\times r}\rightarrow \R^{\ell n\times \ell n\times \ell n}$ as follows
\begin{equation}
  \label{eq:hnu}
  \hnu_{ijk}(Z)=\begin{cases}\sum_{abc}Z_{abc}(\hU_{M_2}\hsm^{1/2})_{ia}(\hU_{M_2}\hsm^{1/2})_{jb}(\hU_{M_2}\hsm^{1/2})_{kc},& \text{ if } \lceil \frac{i}{\ell} \rceil\neq \lceil \frac{j}{\ell} \rceil\neq \lceil \frac{k}{\ell} \rceil\neq \lceil \frac{i}{\ell} \rceil, \\
0, & \mbox{otherwise}.\end{cases}
\end{equation}
\STATE Define $\hA: \R^{r\times r\times r}\rightarrow \R^{r\times r\times r}$ s.t. $\hA(Z)=\hnu(Z)[\hU_{M_2}\hsm^{-1/2}, \hU_{M_2}\hsm^{-1/2}, \hU_{M_2}\hsm^{-1/2}]$
\STATE Output: $\hG=\arg\min_{Z} \|\hA(Z)-\cP_{\Omega_3}(S_3)[\hU_{M_2}\hsm^{-1/2}, \hU_{M_2}\hsm^{-1/2}, \hU_{M_2}\hsm^{-1/2}]\|_F^2$
\end{algorithmic}
\end{algorithm}
% \begin{center}
% \begin{tabular}{ll}
% \hline
% \vspace{-.35cm}\\
% \multicolumn{2}{l}{\Spectral}\\
% \hline
% \vspace{-.35cm}\\
% \multicolumn{2}{l}{{\bf Input:} Samples $\{x_i\}_{i\in\cS}$, $\varepsilon$, $\delta$ }\\
% \multicolumn{2}{l}{{\bf Output:} $\hpi$ $\hwt$}\\
% 1:  & \\
% 2:  & \\
% 3:  & \\
% 4:  & Post-processing to get a valid probability distribution\\
% & \hspace{0.2cm} $\leftarrow $\\
% \hline
% \end{tabular}
% \end{center}
%$\varepsilon_w$ 
%$\varepsilon_\pi$
In this section, we describe the proposed approach in detail and 
provide finite sample performance guarantees for each components: \matrixcompletion {} and \tensorLS.
These results are crucial in proving the finite sample bound in 
Theorem \ref{thm:finite}. 
As mentioned in the previous section, the algorithm 
first estimates $M_2$ using the alternating minimization procedure. 
Recall that the second moment of the data given by $S_2$ cannot 
estimate the block-diagonal entries of $M_2$. That is, even in the case 
of infinite samples, we only have consistency in the off-block-diagonal entries: 
$\cP_{\Omega_2}(S_2)= \cP_{\Omega_2}(M_2)$. 
However, to apply the ``whitening'' operator to the third order tensor 
(see Theorem~\ref{thm:main_comps}) we need to estimate $M_2$. 

In general it is not possible to estimate $M_2$ from $\cP_{\Omega_2}(M_2)$ as one can fill any entries in the block-diagonal entries. Fortunately, we can avoid such a case since $M_2$ is guaranteed to be of rank $r\ll \ell n$. However, even a low-rank assumption is not enough to recover back $M_2$. For example, if $M_2=\e_1\e_1^T$, then $\cP_{\Omega_2}(M_2)=0$ and one cannot recover back $M_2$. Hence, we  make an additional standard assumption that $M_2$ is $\mu$-block-incoherent, where a symmetric rank-$r$ matrix $A$ with singular value decomposition $A=USV^T$ is 
$\mu$-block-incoherent if the operator norm of all $\ell\times \k$ blocks of $U$ are upper bounded by \vspace*{-5pt}
\begin{eqnarray}
\label{eq:defincoherence}
	\big\| U^{(i)} \big\|_2 &\leq& \mu\sqrt{\frac{r}{n}}\;, \text{ for all } i \in[n]\;,\vspace*{-5pt}
\end{eqnarray}
where $U^{(i)}$ is an $\ell\times \k$ sub matrix of $U$ which is defined by the 
block from the $((i-1)\ell +1)$-th row to the $(i \ell )$-th row. 
For a given matrix $M$, the smallest value of $\mu$ that satisfy the above condition is referred to as 
the block-incoherence of $M$.
% (see Section~\ref{sec:notation}). 
%In the context of crowdsourcing application, these two assumptions require that the number of ``workers'' is large and that most of the workers give correct answer with reasonably high probability. 

Now, assuming that $M_2$ satisfies two assumptions, $r\ll \ell n$ and 
$M_2$ is $\mu$-block incoherent, we provide an alternating minimization method that provably recovers $M_2$. 
In particular, we model $M_2$ explicitly using a bi-linear form 
$M_2=\hU^{(t+1)} (U^{(t)})^T$ with variables $\hU^{(t+1)}\in\reals^{\ell\n\times r}$ and $U^{(t)}\in\reals^{\ell\n \times r}$. 
We iteratively solve for $\hU^{(t+1)}$ for fixed $U^{(t)}$, and 
use QR decomposition to orthonormalize $\hU^{(t+1)}$ to get $U^{(t+1)}$.
%and then solve for $\hU$ by solving for $U$. 
%We then update $U$ by QR-decomposition of $\hU$. 
Note that the 
QR-decomposition is {\em not required} for our method but we use 
it only for ease of analysis. Below, we give the precise recovery 
guarantee for the alternating minimization method (Algorithm~\ref{algo:altmin}). 
\begin{theorem}[Matrix completion using alternating minimization]
	\label{thm:matrixam}
	For an $\ell n \times \ell n$ symmetric rank-$r$ matrix $\M$ with block-incoherence $\mu$, 
	we observe off-block-diagonal entries corrupted by noise: \vspace*{-2pt}
	\begin{eqnarray*}
		\hM_{ij} &=& \left\{ \begin{array}{rl} \M_{ij} + E_{ij} & \text{ if } \lceil  \frac{i}{\ell} \rceil \neq \lceil \frac{j}{\ell} \rceil \;, \\ 
			0 & \text{ otherwise.}\end{array} \right.\vspace*{-15pt}
	\end{eqnarray*}
	Let $\hM^{(\tau)}$ denote the output after $\tau$ iterations of \matrixcompletion. %the alternating minimization algorithm, alternating minimization algorithm. 
	If $\mu \leq (\sigma_r(\M)/\sigma_1(\M)) \sqrt{n/(32\,r^{1.5})} $, 
	the noise is bounded by 
	$\|\cP_{\Omega_2}(E)\|_2 \leq \sigma_r(\M)  /32\sqrt{r}$, 
	and each column of the noise is bounded by 
	$\|\cP_{\Omega_2}(E)_{i}\| \leq \sigma_1(\M) \mu \sqrt{3r/(8\,\n\,\ell)}$, $\forall i\in\n\ell$,  
	then after $\tau \geq (1/2)\log\big(2\|\M\|_F/\varepsilon \big)$ iterations of \matrixcompletion, 
	the estimate $\hM^{(\tau)}$ satisfies:\vspace*{-5pt} 
	\begin{eqnarray*}
		\| \M - \hM^{(\tau)} \|_2 &\leq& \varepsilon +  \frac{9\, \|M\|_F \,\sqrt{r}\,}{\sigma_r(M)}\|\cP_{\Omega_2}(E)\|_2 \;, \vspace*{-5pt}
	\end{eqnarray*}
	for any $\varepsilon\in(0,1)$. 
	Further, $\hM^{(\tau)}$ is $\mu_1$-incoherent with $\mu_1=6\mu\sigma_1(M_2)/\sigma_r(M_2)$. 
\end{theorem}
%\begin{theorem}\label{thm:altmin}
%Let $M_2\in \R^{\ell n\times \ell n}$ be a rank-$r$ matrix that is $\mu$-block incoherent as well. Also, let sample second moment estimation be given by: $S_2=P_{\Omega_2}(M_2)+N$, where $N$ is the ``noise'' matrix s.t. $\|N\|_2\leq .1 \sigma_r(M_2)$. Also, let $n\geq \mu^2r^4\left(\frac{\sigma_1(M_2)}{\sigma_r(M_2)}\right)^4$. 
%Let $\hM_2$ be the output of alternating minimization after $T$ steps. Then, 
%$$\|\hM_2-M_2\|_2\leq \epsilon \sigma_r(M_2),$$ 
%where $\epsilon\leq \max(\log(1/T), 16\|N\|_2)$. 
%\end{theorem}
For estimating $M_2$, the noise $E$ in the off-block-diagonal entries are due to insufficient sample size. 
We can precisely bound how large the sampling noise is in the following lemma. 
\begin{lemma}\label{lem:conc_mx}
  Let $S_2=\frac{2}{|\S|}\sum_{t\in \{1,\ldots,|\cS|/2\}} x_t x_t^T$ be the sample co-variance matrix. Also, let 
  $E=\| \cP_{\Omega_2}(S_2)- \cP_{\Omega_2}(M_2)\|_2$. Then, \vspace*{-5pt}
$$\|E\|_2\;\; \leq\;\;8 \sqrt{\frac{n^2\log(n\ell/\delta)}{|\S|}}.\vspace*{-5pt}$$
Moreover, $\|E_i\|_2\leq  8\sqrt{{n\log(1/\delta)}/{|\S|}}$, for all $i\in [n\ell].$
\end{lemma}
\noindent The above theorem  shows that $M_2$ can be recovered exactly from infinite many samples, if $n\geq \frac{\mu^2\sigma_1(M)^2r^{1.5}}{\sigma_r(M)^2}$. Furthermore, using Lemma~\ref{lem:conc_mx}, $M_2$ can be recovered approximately, 
with sample size $|\S|= O({n^2(\ell+r)}/{\sigma_r(M)^2})$. Now, recovering $M_2=\Pi W \Pi^T$ recovers the left-singular space of $\Pi$, i.e., range($U$). However, we still need to recover $W$ and the right-singular space of $\Pi$, i.e., range($V$). 

To this end, we can estimate the tensor $M_3$, ``whiten'' the tensor using $\hum\hsm^{-1/2}$ (recall that, $\hM_2=\hum \hsm \hum^T$), and then use tensor decomposition techniques to solve for $V, W$. However, we show that estimating $M_3$ is not necessary, we can directly estimate the ``whitened'' tensor by solving a system of linear equations. In particular, we design an operator $\hA: \R^{r\times r\times r} \rightarrow \R^{r\times r\times r}$ such that $\hA(\tG)\approx \cP_{\Omega_3}(S_3)[\hum\hsm^{-1/2},\hum\hsm^{-1/2},\hum\hsm^{-1/2}]$, where \vspace*{-5pt}
\begin{equation}
	\tG\;\equiv\;\sum_{q\in [r]} \frac{1}{\sqrt{w_q}}(R_3\e_q\otimes R_3\e_q\otimes R_3\e_q),\text{ and }\;\; 
	R_3\;\equiv\;\hsm^{-1/2}\hum^T \Pi W^{1/2}.\label{eq:tensor_g}\vspace*{-5pt}\end{equation}
Moreover, we show that $\hA$ is nearly-isometric. Hence, we can efficiently estimate $\tG$, using the following system of equations: 
\begin{equation}
	\hG\;=\;\arg\min_{Z}\|\hA(Z)-\cP_{\Omega_3}(S_3)[\hum\hsm^{-1/2},\hum\hsm^{-1/2},\hum\hsm^{-1/2}]\|_F^2.\label{eq:tensor_hg}\end{equation}

Let $\mu$ and $\mu_1$ denote the block-incoherence 
of $M_2$ and $\hM_2$ respectively,  
as defined in \eqref{eq:defincoherence}. 

\begin{theorem}\label{thm:tensorls}
Let $\tG$, $\hG$ be as defined in \eqref{eq:tensor_g}, \eqref{eq:tensor_hg}, respectively. %Also, let 
%$S_3=\frac{1}{|\S|}\sum_{t\in \cS } \sum_{(a,b,c)\in \Omega_3} x_{t,a}x_{t,b} x_{t,c}$. %Let $\hG$ be obtained by solving the following least squares problem:
%$$\hG=\arg\min_{Z}\|\hA(Z)-P_{\Omega_3}(S_3)\left[\hU_{M_2}\hsm^{-1/2},\hU_{M_2}\hsm^{-1/2},\hU_{M_2}\hsm^{-1/2}\right]\|_F^2.$$
If $n\geq 144 r^3 \sigma_1(M_2)^2/\sigma_r(M_2)^2$,  
then the following holds with probability at least $1-\delta$: 
\begin{eqnarray*}
%	\|\hG-\tG\|_F &\leq & \Big(\,  \frac{24\mu_1^3\mu r^{3.5} \sigma_1(M_2)^{3/2} }{n\sqrt{\wmin}\sigma_r(M_2)^{3/2}}\,\varepsilon_{M_2} + \frac{6r^3\mu_1^3}{\sigma_r(M_2)^{3/2}}\sqrt{\frac{2\log(1/\delta)}{|\cS|}}    \,\Big) \;,
	\|\hG-\tG\|_F &\leq & \ \frac{24\mu_1^3\mu r^{3.5} \sigma_1(M_2)^{3/2} }{n\sqrt{\wmin}\sigma_r(M_2)^{3/2}}\,\varepsilon_{M_2} +
		 2 \Big\| \cP_{\Omega_3}(M_3-S_3)[\hum\hsm^{-1/2},\hum\hsm^{-1/2},\hum\hsm^{-1/2}]\Big\|_F  \;,
\end{eqnarray*}
for $\varepsilon_{M_2} \equiv (1/\sigma_r(M_2)) \|\hM_2-M_2\|_2$.
%Moreover, $\|\Pi-\hPi\|_2\leq $. 
\end{theorem}
We can also prove a bound on the sampling noise for the third order tensor in the following lemma. 
\begin{lemma}
\label{lem:conc_tr}
Let $S_3=\frac{2}{|\S|}\sum_{t\in \{|\cS|/2+1,\ldots,|\cS|\}}(x_t\otimes x_t\otimes x_t)$. Then, there exists a positive numerical constant $C$ such that, 
with probability at least $1-\delta$, 
\begin{eqnarray*}
	\Big\| \cP_{\Omega_3}(M_3-S_3)[\hum\hsm^{-1/2},\hum\hsm^{-1/2},\hum\hsm^{-1/2}]\Big\|_F  &\leq&
		\frac{C\, r^3\, \mu_1^3\, \n^{3/2}}{\sigma_r(M_2)^{3/2}}\sqrt{\frac{\log(1/\delta)}{|\cS|}}\;.
\end{eqnarray*}
\end{lemma}

Next, we apply the tensor decomposition method of \cite{AGHKT12} to decompose obtained tensor, $\hG$, and obtain $\hR_3, \hW$ that approximates $R_3$ and $W$. We then use the obtained estimate $\hR_3, \hW$ to estimate $\Pi$; see Algorithm~\ref{algo:main} for the details.
%Now, using $\hG$, the above given estimate of $G$, we use tensor decomposition method of \cite{AGHKT12} to {\em approximately} obtain $R_3$ and $W$, which are then used to estimate $\Pi$; see Algorithm~\ref{algo:main} for the details. 
In particular, using  Theorem~\ref{thm:matrixam} and Theorem~\ref{thm:tensorls}, Algorithm~\ref{algo:main} provides the following estimate for $\Pi$: $$\hPi=\hum\hsm^{1/2}\hR_3\hW^{-1/2}\approx \hum\hum^T\Pi.$$
Now, $\|\hPi-\Pi\|_2$ can be bounded by using the above equation along with the fact that ${\rm range}(\hum)\approx {\rm range}(\Pi)$. See Section~\ref{sec:prf_main} for a detailed proof.

%
%=========================================================================
%
\section{Applications in Crowdsourcing}

Crowdsourcing has emerged as an effective paradigm for solving large-scale data-processing tasks
in domains where humans have an advantage over computers. 
Examples include image classification, video annotation, data entry, optical character recognition, and translation. 
For tasks with discrete choice outputs, 
one of the most widely used model is the Dawid-Skene model 
introduced in \cite{DS79}: each expert $j$ is modeled through a $r \times r$ {\em confusion matrix} $\pi^{(j)}$ 
where $\pi^{(j)}_{pq}$ is the probability that the expert answers $q$ when the true label is $p$. 
This model was developed to study how different clinicians give different diagnosis, even when 
they are presented with the same medical chart. 
This is a special case, with $\ell=\k$, 
of the mixture model studied in this paper. 

Historically, 
a greedy algorithm based on Expectation-Maximization 
has been widely used for inference \cite{DS79,smyth95,HZ98,SPI08}, 
but with no understanding of how the performance 
changes with the problem parameters and sample size. 
Recently, spectral approaches were proposed and analyzed with provable guarantees. 
For a simple case when there are only two labels, i.e. $\k=\ell=2$, 
Ghosh et al. in \cite{Ghosh} and Karger et al. in \cite{KOS11allerton} 
analyzed a spectral approach of using the top singular vector 
for clustering under Dawid-Skene model. 
The model studied in these work is a special case of our model with 
$\k=\ell=2$ and $\wt=[1/2,1/2]$, and $\pi^{(j)} = \begin{bmatrix} p_j& 1-p_j\\1-p_j & p_j\end{bmatrix}$. 
Let $q=(1/n)\sum_{j\in[n]}{2(p_j-1)^2}$, then it follows that $\sigma_1(M_2)=(1/2)n$ and $\sigma_2(M_2)=(1/2)nq$. 
It was proved in \cite{Ghosh,KOS11allerton} that 
if we project each data point $x_i$ onto the second singular  vector of $S_2$ the empirical second moment, 
and make a decision based on the sign of this projection, we get good estimates with 
the probability of misclassification scales as 
$O(1/\sigma_r(M_2))$. 

More recently, Karger et al. in \cite{KOS11OR} proposed a new approach 
based on a message-passing algorithm for computing the top singular vectors, and improved this 
misclassification bound to an exponentially decaying $O(e^{-C\sigma_r(M_2)})$ for some positive numerical constant $C$.
However, these approaches highly rely on the fact that there are only two ground truth labels, and the algorithm and analysis cannot be generalized. 
These spectral approaches has been extended to general $\k$ in \cite{KOS13SIGMETRICS} 
with misclassification probability scaling as $O(\k/\sigma_r(M_2))$, 
but this approach still uses the existing binary classification algorithms as a black box and tries to solve 
a series of binary classification tasks.

Furthermore, existing spectral approaches 
use $S_2$ directly for inference. 
This is not consistent, since even if infinite number of samples are provided, 
this empirical second moment does not converge to $\M_2$. 
Instead, we use recent developments in matrix completion to 
recover $M_2$ from samples, thus providing a consistent estimator. 
Hence, we provide a robust 
clustering algorithm for crowdsourcing and provide estimates for the mixture distribution with provable guarantees. 
Corollary \ref{coro:cluster} shows that with large enough samples, the misclassification probability of our approach scales as 
$O(re^{-C(r\,\sigma_r(\M_2)^2/n)})$ for some positive constant $C$. 
This is an exponential decay and a significant improvement over the known 
error bound of $O(\k/\sigma_r(M_2))$.

\section{Conclusion}
We presented a method for learning a mixture of $\ell$-wise discrete distribution with distribution parameters $\Pi, W$. Our method shows that assuming $n\geq C  r^3 \kappa^{4.5}$ 
and the number of samples to be $|S|\geq C_1 (n\,r^7\,\kappa^9 \log (n/\delta))/(w_{min}^2\varepsilon_{\Pi}^2)$, 
we have $\|\hPi-\Pi\|_2\leq \varepsilon_{\Pi}$ where $\kappa = {\sigma_1(M_2)}/{\sigma_r(M_2)}$,  and $M_2=\Pi W\Pi^T$. 

Note that our algorithm does not require any separability condition on the distribution, is consistent for infinite samples, and is robust to noise as well. That is, our analysis can be easily extended to the noisy case, where there is a small amount of noise in each sample. %even though the sample complexity of our algorithm is inversely dependent on $w_{min}$However, our result depends on $w_{min}$ and the number of samples required increases as $w_{min}$ decreases. 

Our sample complexity bounds include the condition number of the distribution $\kappa$ which implies that our method requires $\kappa$ to be at most $poly(\ell, r)$. This makes our method unsuitable for the problem of learning Boolean functions \cite{FOS08}. However, it is not clear if is possible to design an efficient algorithm with sample complexity independent of the condition number.  We leave further study of the dependence of sample complexity on the condition number as a topic for future research. %the number of samples should depend polynomially on the condition number. It is an open problem to either have better sample complexity analysis w.r.t. $\kappa$ or show that the bound required by our method is optimal. 

Another drawback of our method is that $n$ is required to be $n=\Omega(r^3)$. We believe that this condition is natural, as one cannot recover the distribution for $n=1$. However, establishing tight information theoretic lower bound on  $n$ (w.r.t. $\ell, r$) is still an open problem. %it is an open problem to establish  so that the  distribution parameters can be recovered from the samples. 

For the crowdsourcing application, the current error bound for clustering translates into  
$O(e^{-Cn q^2})$ when $r=2$. This is not as strong as the best known error bound of $O(e^{-Cnq})$, 
since $q$ is always less than one. 
The current analysis and algorithm for clustering needs to be improved 
to get an error bound of $O(re^{-Cr\sigma_r(M_2)})$ for general $r$ 
such that it gives optimal error rate for the special case of $r=2$.

The sample complexity also depends on $1/\wmin$, which we believe is unnecessary. 
If there is a component with small  mixing weight, we should be able to 
ignore such component smaller than the sample noise level and still guarantee the same level accuracy. 
To this end, we need an adaptive algorithm that detects the number of components that are non-trivial and this is a subject of future research. 

More fundamentally, all of the moment matching methods based on the spectral decompositions 
suffer from the same restrictions. It is required that the underlying tensors have rank equal to the number of components, 
and the condition number needs to be small. 
However, the problem itself is not necessarily more difficult when the condition number is larger. 

Finally, we believe that our technique of completion of the second and the higher order moments should have application to several other mixture models that involve $\ell$-wise distributions, e.g., mixed membership stochastic block model with $\ell$-wise connections between nodes. 

\bibliographystyle{amsalpha}

\bibliography{mturk}

\newcommand{\etalchar}[1]{$^{#1}$}
\providecommand{\bysame}{\leavevmode\hbox to3em{\hrulefill}\thinspace}
\providecommand{\MR}{\relax\ifhmode\unskip\space\fi MR }
% \MRhref is called by the amsart/book/proc definition of \MR.
\providecommand{\MRhref}[2]{%
  \href{http://www.ams.org/mathscinet-getitem?mr=#1}{#2}
}
\providecommand{\href}[2]{#2}
\begin{thebibliography}{KMR{\etalchar{+}}94}

\bibitem[AGH{\etalchar{+}}12]{AGHKT12}
Anima Anandkumar, Rong Ge, Daniel Hsu, Sham~M. Kakade, and Matus Telgarsky,
  \emph{Tensor decompositions for learning latent variable models}, CoRR
  \textbf{abs/1210.7559} (2012).

\bibitem[AGM12]{AroraGM12}
Sanjeev Arora, Rong Ge, and Ankur Moitra, \emph{Learning topic models - going
  beyond {SVD}}, FOCS, 2012, pp.~1--10.

\bibitem[AGMS12]{AroraGMS12}
Sanjeev Arora, Rong Ge, Ankur Moitra, and Sushant Sachdeva, \emph{Provable
  {ICA} with unknown {G}aussian noise, with implications for {G}aussian
  mixtures and autoencoders}, NIPS, 2012, pp.~2384--2392.

\bibitem[AHK12]{AHK12}
A.~Anandkumar, D.~Hsu, and S.~M. Kakade, \emph{A method of moments for mixture
  models and hidden markov models}, arXiv preprint arXiv:1203.0683 (2012).

\bibitem[AK01]{AK01}
S.~Arora and R.~Kannan, \emph{Learning mixtures of arbitrary {G}aussians},
  STOC, 2001, pp.~247--257.

\bibitem[AM05]{AM05}
Dimitris Achlioptas and Frank McSherry, \emph{On spectral learning of mixtures
  of distributions}, Learning Theory, Springer, 2005, pp.~458--469.

\bibitem[CHRZ07]{CHR07}
Kamalika Chaudhuri, Eran Halperin, Satish Rao, and Shuheng Zhou, \emph{A
  rigorous analysis of population stratification with limited data}, SODA,
  2007, pp.~1046--1055.

\bibitem[CR08]{CR08}
K.~Chaudhuri and S.~Rao, \emph{Learning mixtures of product distributions using
  correlations and independence.}, COLT, 2008, pp.~9--20.

\bibitem[DS79]{DS79}
A.~P. Dawid and A.~M. Skene, \emph{Maximum likelihood estimation of observer
  error-rates using the em algorithm}, Journal of the Royal Statistical
  Society. Series C (Applied Statistics) \textbf{28} (1979), no.~1, 20--28.

\bibitem[FM99]{FM99}
Y.~Freund and Y.~Mansour, \emph{Estimating a mixture of two product
  distributions}, COLT, 1999, pp.~53--62.

\bibitem[FOS08]{FOS08}
J.~Feldman, R.~O'Donnell, and R.~A Servedio, \emph{Learning mixtures of product
  distributions over discrete domains}, SIAM Journal on Computing \textbf{37}
  (2008), no.~5, 1536--1564.

\bibitem[GAGG13]{GAGG13}
Suriya Gunasekar, Ayan Acharya, Neeraj Gaur, and Joydeep Ghosh, \emph{Noisy
  matrix completion using alternating minimization}, Machine Learning and
  Knowledge Discovery in Databases, Springer, 2013, pp.~194--209.

\bibitem[GKM11]{Ghosh}
A.~Ghosh, S.~Kale, and P.~McAfee, \emph{Who moderates the moderators?:
  crowdsourcing abuse detection in user-generated content}, EC, 2011,
  pp.~167--176.

\bibitem[GR12]{GoyalL12}
Navin Goyal and Luis Rademacher, \emph{Efficient learning of simplices}, CoRR
  \textbf{abs/1211.2227} (2012).

\bibitem[HK13]{HK13}
D.~Hsu and S.~M. Kakade, \emph{Learning mixtures of spherical {G}aussians:
  moment methods and spectral decompositions}, ITCS, 2013, pp.~11--20.

\bibitem[HKZ12]{HKZ12a}
D.~Hsu, S.~M. Kakade, and T.~Zhang, \emph{A spectral algorithm for learning
  hidden markov models}, Journal of Computer and System Sciences \textbf{78}
  (2012), no.~5, 1460--1480.

\bibitem[HZ98]{HZ98}
Siu~L Hui and Xiao~H Zhou, \emph{Evaluation of diagnostic tests without gold
  standards}, Statistical methods in medical research \textbf{7} (1998), no.~4,
  354--370.

\bibitem[JNS13]{JNS13}
Prateek Jain, Praneeth Netrapalli, and Sujay Sanghavi, \emph{Low-rank matrix
  completion using alternating minimization}, STOC, 2013, pp.~665--674.

\bibitem[KMR{\etalchar{+}}94]{KMR94}
M.~Kearns, Y.~Mansour, D.~Ron, R.~Rubinfeld, R.~E. Schapire, and L.~Sellie,
  \emph{On the learnability of discrete distributions}, STOC, 1994,
  pp.~273--282.

\bibitem[KOS11a]{KOS11allerton}
D.~R. Karger, S.~Oh, and D.~Shah, \emph{Budget-optimal crowdsourcing using
  low-rank matrix approximations}, Allerton, 2011.

\bibitem[KOS11b]{KOS11OR}
\bysame, \emph{Budget-optimal task allocation for reliable crowdsourcing
  systems}, arXiv preprint arXiv:1110.3564 (2011).

\bibitem[KOS13]{KOS13SIGMETRICS}
\bysame, \emph{Efficient crowdsourcing for multi-class labeling}, Proceedings
  of the ACM SIGMETRICS/international conference on Measurement and modeling of
  computer systems, 2013, pp.~81--92.

\bibitem[McS01]{McS01}
Frank McSherry, \emph{Spectral partitioning of random graphs}, FOCS, 2001,
  pp.~529--537.

\bibitem[MV10]{MoitraV10}
Ankur Moitra and Gregory Valiant, \emph{Settling the polynomial learnability of
  mixtures of {G}aussians}, FOCS, 2010, pp.~93--102.

\bibitem[RSS12]{RabaniSS12}
Yuval Rabani, Leonard~J. Schulman, and Chaitanya Swamy, \emph{Learning mixtures
  of arbitrary distributions over large discrete domains}, CoRR
  \textbf{abs/1212.1527} (2012).

\bibitem[SCPW12]{SaundersonCPW12}
James Saunderson, Venkat Chandrasekaran, Pablo~A. Parrilo, and Alan~S. Willsky,
  \emph{Diagonal and low-rank matrix decompositions, correlation matrices, and
  ellipsoid fitting}, SIAM J. Matrix Analysis Applications \textbf{33} (2012),
  no.~4, 1395--1416.

\bibitem[SFB{\etalchar{+}}95]{smyth95}
P.~Smyth, U.~Fayyad, M.~Burl, P.~Perona, and P.~Baldi, \emph{{Inferring ground
  truth from subjective labelling of venus images}}, NIPS, 1995,
  pp.~1085--1092.

\bibitem[SPI08]{SPI08}
V.~S. Sheng, F.~Provost, and P.~G. Ipeirotis, \emph{Get another label?
  improving data quality and data mining using multiple, noisy labelers}, KDD,
  2008, pp.~614--622.

\bibitem[SRH07]{SRH07}
S.~Sridhar, S.~Rao, and E.~Halperin, \emph{An efficient and accurate
  graph-based approach to detect population substructure}, Research in
  Computational Molecular Biology, 2007, pp.~503--517.

\bibitem[TM10]{TM10}
Dan-Cristian Tomozei and Laurent Massouli{\'e}, \emph{Distributed user
  profiling via spectral methods}, ACM SIGMETRICS Performance Evaluation
  Review, vol.~38, 2010, pp.~383--384.

\bibitem[Tro12]{Tropp12}
Joel~A Tropp, \emph{User-friendly tail bounds for sums of random matrices},
  Foundations of Computational Mathematics \textbf{12} (2012), no.~4, 389--434.

\bibitem[VW04]{VempalaW04}
Santosh Vempala and Grant Wang, \emph{A spectral algorithm for learning mixture
  models}, J. Comput. Syst. Sci. \textbf{68} (2004), no.~4, 841--860.

\end{thebibliography}
\newpage
\appendix 
\section*{Appendix}
\section{Proofs}
\label{sec:prf}
In this section, we give detailed proofs for all the key theorems/lemmata that we require to prove our main result (Theorem~\ref{thm:matrixam}, Theorem~\ref{thm:tensorls}). 

\subsection{Proof of Theorem~\ref{thm:matrixam}}
We analyze each iteration and show that we get closer to the optimal solution up to  
a certain noise level at each step. 
To make the block structures explicit, we use index $(i,a)$ for some $i\in[n]$ and $a\in[\ell]$ 
to denote $(i-1)\ell+a \in [\ell n]$. The least squares update gives: 
\begin{eqnarray*}
	U^{(t+1)} &=& \arg\min_{V\in\reals^{\ell n\times \ell n}} \sum_{i,j\in[n], a,b\in[\ell], i\neq j} \Big( \hM_{(i,a),(j,b)} - \big(V(\hU^{(t)})^T\big)_{(i,a),(j,b)} \Big)^2. 
\end{eqnarray*}
Setting the gradient to zero, we get: 
\begin{eqnarray*}
	 -2 \sum_{j\neq i, b\in[\ell]} \Big( \M_{(i,a),(j,b)} +E_{(i,a),(j,b)} - \big\< U^{(t+1)}_{(i,a)}, \hU^{(t)}_{(j,b)}\big\>  \Big) \hU^{(t)}_{(j,b)}  &=& 0\;,
\end{eqnarray*}
for all $i\in[n]$ and $a\in[\ell]$. 
Here, $U^{(t)}_{(j,b)}$ is a $\k$-dimensional column vector representing the ${((j-1)\ell+b)}$-th row of $U^{(t)}$.  
Let $\M = USU^T$ be the singular value decomposition of $\M$.
The $\k$-dimensional column vector $U^{(t+1)}_{(i,a)}$  can be written as:   
\begin{eqnarray}
	U^{(t+1)}_{(i,a)} &=& (B^{(i,a)})^{-1}\, C^{(i,a)} \,S  \,U_{(i,a)} \,+\, (B^{(i,a)})^{-1}\,N_{(i,a)} \nonumber\\ 
		&=& \underbrace{D\,S\,U_{(i,a)}}_{\text{power iteration}} \,-\, \underbrace{(B^{(i,a)})^{-1}\, \big( B^{(i,a)} D \,-\, C^{(i,a)} \big) \,S  \,U_{(i,a)}}_{\text{error due to missing entries}} \,+\, \underbrace{(B^{(i,a)})^{-1}\, N_{(i,a)}}_{\text{error due to noise}} \;,   \label{eq:amupdaterule}
\end{eqnarray}
where, 
\begin{eqnarray*}
	B^{(i,a)} &=& \sum_{j\neq i, j\in[n], b\in[\ell]} \hU^{(t)}_{(j,b)} (\hU^{(t)})^T_{(j,b)}  \, \in \, \reals^{\k\times \k} \\
	C^{(i,a)} &=& \sum_{j\neq i, j\in[n], b\in[\ell]} \hU^{(t)}_{(j,b)} U^T_{(j,b)}  \, \in \, \reals^{\k\times \k} \\
	D 	      &=& \sum_{j\in[n], b\in[\ell]} \hU^{(t)}_{(j,b)} U^T_{(j,b)}  \, \in \, \reals^{\k\times \k} \\
	N_{(i,a)} &=& \sum_{j\neq i, j\in[n], b\in[\ell]} E_{(i,a),(j,b)} \hU^{(t)}_{(j,b)} \,\in\,\reals^{\k\times 1}\;.
\end{eqnarray*}
Note that, the above quantities are independent of index $a$, but we carry the index for uniformity of notation. 

In a matrix form of dimension $\ell n \times \k$, 
we use $\Fmiss\in\reals^{\ell n \times \k}$ to denote the error due to missing entries 
and $\Fnoise\in\reals^{\ell n \times \k}$ to denote the error due to the noise 
such that  
\begin{eqnarray}
	U^{(t+1)}  &=& \M \,\hU^{(t)} - \Fmiss^{(t+1)} + \Fnoise^{(t+1)} \;, \text{ and } \nonumber\\ 
	\hU^{(t+1)} &=& \Big( \M \,\hU^{(t)} - \Fmiss^{(t+1)} + \Fnoise^{(t+1)} \Big) \big( R_U^{(t+1)}\big)^{-1} \;,   
	\label{eq:ammatrixstep}
\end{eqnarray} 
where we define $R_U^{(t+1)}$ to be the upper triangular matrix obtained by QR decomposition of 
$U^{(t+1)}=\hU^{(t+1)} R_U^{(t+1)}$.
The explicit formula for $\Fmiss$ and $\Fnoise$ is given in \eqref{eq:degFmiss} and \eqref{eq:degFnoise}. 
Then, the error after $t$ iterations of the alternating minimization is bounded by
\begin{eqnarray}
	\big\| \,\M - \hU^{(t)} \, \big(U^{(t+1)}\big)^T\,\big\|_F &\leq&  \big\| \,(\id-\hU^{(t)}\big(\hU^{(t)}\big)^T) \,U\,S\,\big\|_F + \big\| \Fmiss^{(t+1)}\big\|_F + \big\|\Fnoise^{(t+1)} \big\|_F \;. \label{eq:amerror1}
\end{eqnarray}
Let $U_\perp\in\reals^{\ell n \times (\ell n-\k)}$ be an orthogonal matrix 
spanning the  
subspace orthogonal to $U$.
We use the following definition of distance 
between two $\k$-dimensional subspaces in $\reals^{\ell n}$. 
\begin{eqnarray*}
	d(\hU,U) = \big\| \,U^T_\perp \, \hU \, \big\|_2 \;.
\end{eqnarray*}
The following key technical lemma provides upper bounds on each of the error terms in \eqref{eq:amerror1}. 
\begin{lemma}
	\label{lem:ammatrixbound}
%	Under the hypotheses of Lemma \ref{lem:matrixam}, 
%	If $d(U^{(t)},U)\leq 1/2$ and  
	For any $\mu_1$-incoherent orthogonal matrix $U^{(t)}\in\reals^{\ell\n\times\k}$ 
	and $\mu$-incoherent matrix $\M\in\reals^{\ell\n\times\ell\n}$, 
	% has incoherence upper bounded by $\mu_1$ \leq 6(1+\delta)\mu/(1-\delta)$, then 
	the error after one step of alternating minimization is upper bounded by 
	\begin{eqnarray*}
		\|\Fmiss^{(t+1)}\|_F &\leq& \frac{\sigma_1(\M)\k^{1.5}\mu\mu_1}{n(1-\frac{\mu_1^2 \k}{n})}\, d(\hU^{(t)},U) \;, \\
		\|\Fnoise^{(t+1)}\|_F &\leq&  \frac{1}{1-\frac{\mu_1^2\k}{\n}}\, \sqrt{\k}\, \|\cP_\Omega(E)\|_2 \;, \\
%		\|(R_U^{(t+1)})^{-1}\|_2 &\leq& \frac{1}{\sigma_k(\M)\,\sqrt{1-d(U^{(t)},U)}  - \|\Fmiss\|_2 - \|\Fnoise\|_2 }\;.
	\end{eqnarray*}
	where $\sigma_i(\M)$ is the $i$-th singular value of $\M$. 
\end{lemma}
We show in Lemma  \ref{lem:amincoherence} that
the incoherence assumption is satisfied for all $t$ with $\mu_1=6(\sigma_1(M)/\sigma_\k(M))\mu$.
For $\mu_1\leq \sqrt{n/2\k}$ as per our assumption  
and substituting these bounds into \eqref{eq:amerror1}, we get 
\begin{eqnarray*}
	\big\| \,\M - \hU^{(t)} \, \big(U^{(t+1)}\big)^T\,\big\|_F &\leq& 
		\|\M \|_F\,d(\hU^{(t)},U) \,+\,
		\frac{12\, \sigma_1(\M)^2\,\k^{1.5}\,\mu^2}{n\,\sigma_\k(M)} \,d(\hU^{(t)},U) \,+\, 
		2 \, \sqrt{\k}\, \|\cP_\Omega(E)\|_2 \;, 
\end{eqnarray*}
where the first term follows from the fact that 
$\|(\id-\hU^{(t)}(\hU^{(t)})^T)U\|_2 = \|\hU^{(t)}_\perp(\hU^{(t)}_\perp)^TU\|_2= d(\hU^{(t)},U)$. 
To further bound the distance $d(\hU^{(t)},U)$, 
we first claim that after $t$ iterations of the alternating minimization algorithm, 
the estimates satisfy 
\begin{eqnarray}
		d(\hU^{(t)},U) &\leq& \frac{\varepsilon}{2\|\M\|_F} \, + \,
		\frac{2\sqrt{3\,\k}} {\sigma_\k(\M)}\,\|\cP_\Omega(E)\|_2 \;, \label{eq:amdistance}
\end{eqnarray}
for $t\geq (1/2) \log \big(2\|M\|_F/\varepsilon\big)$. 
For $\mu  \leq \sqrt{n\,\sigma_\k(M)/(12 \k \sigma_1(M))}$ as per our assumption, this gives 
\begin{eqnarray*}
	\big\| \,\M - \hU^{(t)} \, \big(U^{(t+1)}\big)^T\,\big\|_F &\leq& \varepsilon \,+\, 
		\frac{9\,\|M\|_F\, \sqrt{\k}}{\sigma_\k(M)}\, \|\cP_\Omega(E)\|_2 \;. 
\end{eqnarray*}
This proves the desired error bound of Theorem \ref{thm:matrixam}.  

Now, we are left to prove \eqref{eq:amdistance} 
for $t\geq (1/2) \log \big(2\|M\|_F/\varepsilon\big)$. 
This follows from the analysis of each step of the algorithm, 
which shows that we improve at each step up to a certain noise level. 
Define $R_U^{(t+1)}$ to be the upper triangular matrix obtained by QR decomposition of 
$U^{(t+1)}=\hU^{(t+1)} R_U^{(t+1)}$.
Then we can represent the distance using \eqref{eq:ammatrixstep} as: 
	\begin{eqnarray*}
		d(\hU^{(t+1)},U) &=& \Big\| U_\perp^T \big( USU^T \,U^{(t)} - \Fmiss^{(t+1)} + \Fnoise^{(t+1)} \big) \big( R_U^{(t+1)}\big)^{-1} \Big\|_2, \\
		&\leq& \Big(\, \|\Fmiss^{(t+1)}\|_2 \,+ \, \|\Fnoise^{(t+1)}\|_2 \,\Big) \, \big\| \big( R_U^{(t+1)}\big)^{-1}\big\|_2, \\
			&\leq& \frac{12\sqrt{3}\,\sigma_1(M)^2\,\k^{1.5}\,\mu^2 }{\sigma_\k(M)^2\,n} \,d(\hU^{(t)},U) \,+\, 
			\frac{2\,\sqrt{3\k}}{\sigma_\k(M)} \,\|\cP_\Omega(E)\|_2\;,
	\end{eqnarray*}
where we used Lemma \ref{lem:amQR} to bound $\big\| \big( R_U^{(t+1)}\big)^{-1}\big\|_2$,  
Lemma \ref{lem:ammatrixbound} to bound $\|\Fmiss^{(t+1)}\|_2$ and $\|\Fnoise^{(t+1)}\|_2$, 
and Lemma \ref{lem:amincoherence} to bound $\mu_1$.  
For $\mu\leq  \sqrt{n\sigma_\k(M) / (10\,\k^{1.5}\sigma_1(M))}$ as per our assumption, it follows that 
	\begin{eqnarray*}
		d(\hU^{(t)},U) &=& \Big(\frac14\Big)^{t} \,d(\hU^{(0)},U) \,+\, 
			\frac{2\,\sqrt{3\k}}{\sigma_\k(M)} \,\|\cP_\Omega(E)\|_2\;,
	\end{eqnarray*}
Taking $t\geq (1/2) \log \big(2\|M\|_F /\varepsilon\big)$, this finishes the proof of  
the desired bound in \eqref{eq:amdistance}. 

Now we are left to prove that 
starting from a good initial guess we obtain using a simple Singular Value Decomposition(SVD), 
the estimates at every iterate $t$ is incoherent with bounded $\|(R_U^{(t+1)})^{-1}\|_2$.
We first state the following two lemmas upper bounding $\mu_1$ and 
$\|(R_U^{(t+1)})^{-1}\|_2$. Then we prove that the hypotheses of the lemmas 
are satisfied, if we start from a good initialization. 
\begin{lemma}
	Assume that $U$ is $\mu$-incoherent with 
	$\mu\leq (\sigma_\k(M)/\sigma_1(M))\sqrt{n/(32\k^{1.5})} $, 
	$d(\hU^{(t)},U) \leq 1 / 2$, and  
	$\|\cP_\Omega(E)\|_2 \leq \sigma_\k(M)/(16\sqrt{\k})$. 
	Then, 
	\begin{eqnarray*}
		\|(R_U^{(t+1)})^{-1}\|_2 &\leq& \frac{\sqrt{3}}{\sigma_\k(M)}\;. 
	\end{eqnarray*}
	\label{lem:amQR}
\end{lemma}

\begin{lemma}[Incoherence of the estimates]
	\label{lem:amincoherence}
	Assume that $\hU^{(t)}$ is $\tmu$-incoherent with 
	$\tmu\leq \sqrt{n/(2\k)}$, and $U$ is $\mu$-incoherent with 
	$\mu\leq (\sigma_\k(M)/\sigma_1(M))\sqrt{n/(32\k)} $, 
	and the noise $E$ satisfy $\|\cP_\Omega(E)_{(i,a)}\| \leq \sigma_1(M) \mu \,\sqrt{3\k/(8n\ell)}$ 
	for all $i\in[\n]$ and $a\in[\ell]$.
	Then, $\hU^{(t+1)}$ is $\mu_1$-incoherent with  
	\begin{eqnarray*}
		\mu_1 &=&  \frac{6\, \mu\,\sigma_1(\M)}{\sigma_\k(\M)}\;.
	\end{eqnarray*} 
\end{lemma}
For the above two lemmas to hold, 
we need a good initial guess $\hU^{(0)}$ with incoherence less than $4\mu$ 
and error upper bounded by $d(\hU^{(0)},U)\leq 1/2$. Next lemmas shows that we can get 
such a good initial guess by singular value decomposition and truncation. 
And this finishes the proof of Theorem \ref{thm:matrixam}.

\begin{lemma}[Bound on the initial guess]
	\label{lem:aminitial}
	Let $\hU^{(0)}$ be the output of step 3 in the alternating minimization algorithm, 
	and let $\mu_0$ be the incoherence of $\hU^{(0)}$. 
	Assuming $\mu\leq\sqrt{\sigma_\k(M)\,\n/(32\,\sigma_1(M)\,\k^{1.5})}$ and 
	$\|\cP_\Omega(E)\|_2 \leq \sigma_\k(M)/(32 \sqrt{\k}) $, 
	we have the following upper bound on the error and the incoherence: 
	\begin{eqnarray*}
		d(\hU^{(0)},U) &\leq& \frac12, \\
		\mu_0 &\leq& 4 \,\mu. \\
	\end{eqnarray*}
\end{lemma}

%================================================================================
\subsubsection{Proofs of Lemmas ~\ref{lem:ammatrixbound}, \ref{lem:amQR}, \ref{lem:amincoherence}, \ref{lem:aminitial}}
\label{sec:ammatrixbound}
\begin{proof}[Proof of Lemma~\ref{lem:ammatrixbound}]
First, we prove the following upper bound for $\mu_1$-incoherent $\hU^{(t+1)}$. 
\begin{eqnarray*}
	\|\Fmiss^{(t+1)}\|_F &\leq& \frac{\sigma_1(\M)\k^{1.5}\mu\mu_1}{n(1-\frac{\mu_1^2 \k}{n})} d(\hU^{(t)},U)  \;.
\end{eqnarray*} 
We drop the time index $(t+1)$ whenever it is clear from the context, 
to simplify notations. 
Let $\F_{(i,a)}\in\reals^{\k}$ be a column vector 
representing the $(\ell(i-1)+a)$-th row of $\Fmiss\in\reals^{\ell n \times \k}$.
We know from \eqref{eq:amupdaterule} that 
\begin{eqnarray}
	\F_{(i,a)} &=& (B^{(i,a)})^{-1} \, \underbrace{\big( B^{(i,a)} D - C^{(i,a)} \,\big)}_{\equiv\,H^{(i)}}\, S \, U_{(i,a)}\;,
	\label{eq:degFmiss}
\end{eqnarray}
where we define $H^{(i)}\equiv B^{(i,a)} D - C^{(i,a)}$. 
Notice that we dropped $a$ from the index to emphasize that 
$B^{(i,a)}$ and $C^{(i,a)}$ do not depend on $a$. 
\begin{eqnarray*}
	 \|\Fmiss\|_F 
			&\leq& \sqrt{\sum_{i,a} \|(B^{(i,a)})^{-1}\|_2^2 \, \| H^{(i)}\, S \, U_{(i,a)}\|^2} \\ 
			&=&  \max_{j,b} \|(B^{(j,b)})^{-1}\|_2 \, \max_{x\in\reals^{\ell n\times \k}, \|x\|_F=1} \sum_{i\in[n],a\in[\ell],q\in[\k]} 
			x_{(i,a),q}\, e_q^T\, H^{(i)}\, S \, U_{(i,a)} \;. 
\end{eqnarray*}
To upper bound the first term, notice that 
$ \|(B^{(j,b)})^{-1}\|_2 \leq 1/\sigma_\k(B^{(j,b)}) $. 
Since $B^{(j,b)} = \id_{\k\times \k} - \sum_{a\in[\ell]}\hU_{(j,a)}(\hU_{(j,a)})^T$, 
and by incoherence property from Lemma \ref{lem:amincoherence}, 
we have 
\begin{eqnarray}
	\| (B^{(j,b)} )^{-1}\|_2 \leq \frac{1}{1- \frac{\mu_1^2 \k}{n}}\;, \label{eq:boundB}
\end{eqnarray}
for all $(j,b)$.

The second term can be bounded using Cauchy-Schwarz inequality: 
\begin{eqnarray*}
	 \sum_{i\in[n],a\in[\ell],q\in[\k]} 
			x_{(i,a),q}\, e_q^T\, H^{(i)}\, S \, U_{(i,a)} &=& \sum_{i\in[n],q,p\in[\k]}  \big(\sum_{a\in[\ell]} S_p \, U_{(i,a),p} x_{(i,a),q} \big)\, \big(e_q^T\, H^{(i)}e_p\big) \\
			 &\leq& 
	 	\sqrt{ \sum_{i,p,q}\big(\sum_{a\in[\ell]} S_p \, U_{(i,a),p} x_{(i,a),q} \big)^2} \sqrt{ \sum_{i,p,q}( e_q^T\, H^{(i)}e_p)^2 } \;,
\end{eqnarray*}
where $S_p$ is the $p$-th eigenvalue of $\M$. 
Applying Cauchy-Schwarz again, and by the incoherence of $U$ is and $\|x\|=1$, 
\begin{eqnarray}
	\sum_{i,p,q}\big(\sum_{a\in[\ell]} S_p \, U_{(i,a),p} x_{(i,a),q} \big)^2 &\leq& 
	\sum_{i,p,q}S_p^2 \big(\sum_{a\in[\ell]}  \, U_{(i,a),p}^2 \sum_{b\in[\ell]} x_{(i,b),q}^2 \big) \nonumber\\
		&\leq& \sigma_1(\M)^2\, \frac{\mu^2\, \k }{ n } \;. \label{eq:ammatrixbound1}
\end{eqnarray}
%Let $U_p$ denote the $p$-th column of $U$, and 
%$U_{(i)} $ denote an 
%$\ell\times\k$ dimensional sub-matrix  of $U$ from 
%the $(\ell(i-1)+1)$ to the $(\ell i)$-th row of $U$.  Then, 
\begin{eqnarray}
	\sum_{i,p,q}( e_q^T\, H^{(i)}e_p)^2 %&=& \sum_{i,p,q}  \Big(\hU_{(i),q}^T \big(\, U_{(i),p} - \hU_{(i)}^T\hU^T U_p \,\big)\Big)^2 \nonumber\\ 
		&=&  \sum_{i,p,q}  \Big(\sum_a \hU_{(i,a),q} \big(\, U_{(i,a),p} - \hU_{(i,a)}^T\hU^T U_p \,\big)\Big)^2 \nonumber\\ 
		&\leq& \sum_i \Big\{ \sum_{a,q} \hU_{(i,a),q}^2 \sum_{b,p} \big(\, U_{(i,b),p} - \hU_{(i,b)}^T\hU^T U_p \,\big)^2 \Big\} \nonumber\\
		&\leq& \frac{\mu_1^2 \,\k }{n}\sum_{i,b,p} \big(\, U_{(i,b),p} - \hU_{(i,b)}^T\hU^T U_p \,\big)^2 \nonumber\\
		&\leq& \frac{\mu_1^2 \,\k }{n} \big(\, \k - {\rm Tr}( U^T \hU \hU^T U) \,\big)^2 \nonumber\\
		&\leq& \frac{\mu_1^2 \,\k^2 \, d(\hU,U)^2 }{n}  \label{eq:ammatrixbound2}\;,
\end{eqnarray}
where the last inequality follows from the fact that 
$d(\hU,U)^2 = \| \hU_\perp^TU\|_2^2 = \|U^T \hU_\perp\hU_\perp^T U\|_2 = \|\id_{\k\times \k} - U^T\hU\hU^TU \|_2 = 1-\sigma_\k(\hU^TU)^2 \geq 1-(1/\k) \sum_p \sigma_p(\hU^TU)^2$. 

Now, we prove an upper bound on $\|\Fnoise^{(t+1)}\|_F$. 
Again, we drop the time index ${(t+1)}$ or $(t)$ whenever it is clear from the context. 
Let $\tF_{(i,a)}\in\reals^{\k}$ denote a column vector representing  the $(\ell(i-1)+a)$-th row of $\Fnoise$. 
We know from \eqref{eq:amupdaterule} that 
\begin{eqnarray}
	\tF_{(i,a)} &=& (B^{(i,a)})^{-1} \,\Big( \hU^TE_{(i,a)} - \sum_{b\in[\ell]} E_{(i,a)(i,b)} \hU_{i,b} \Big)\;,
	\label{eq:degFnoise}
\end{eqnarray}
where $E_{(i,a)}\in\reals^{\ell n}$ is an column vector representing the $(\ell(i-1)+a)$-th row of $E$. 
Then, 
\begin{eqnarray*}
 	\|\Fnoise\|_F 
	&\leq&  \sqrt{\sum_{i\in[\n],a\in[\ell]} \big\|(B^{(i,a)})^{-1}\big\|_2^2\,\Big\| \hU^TE_{(i,a)} - \sum_{b\in[\ell]} E_{(i,a)(i,b)} \hU_{i,b} \Big\|^2} \\
	&\leq& \max_{i,a} \big\|(B^{(i,a)})^{-1}\big\|_2 \, \|\cP_\Omega(E)\hU\|_F\\
	&\leq& \frac{1}{1-\frac{\mu_1^2\k}{\n}}\, \sqrt{\k}\, \|\cP_\Omega(E)\|_2  \;,
\end{eqnarray*}
where $\cP_\Omega$ is the projection onto the sampled entries defined in \eqref{eq:defcP}, 
and we used \eqref{eq:boundB} to bound $\|(B^{(i,a)})^{-1}\|_2$. 
\end{proof}
\begin{proof}[Proof of Lemma \ref{lem:amQR}]
From Lemma 7 in \cite{GAGG13}, we know that  
	\begin{eqnarray*}
		\|(R_U^{(t+1)})^{-1}\|_2 &\leq& \frac{1}{\sigma_\k(\M)\,\sqrt{1-d^2(U^{(t)},U)}  - \|\Fmiss^{(t+1)}\|_2 - \|\Fnoise^{(t+1)}\|_2 }\;.
	\end{eqnarray*}
From Lemma \ref{lem:ammatrixbound} with 
$\mu \leq (\sigma_\k(M)/(6\sigma_1(M)))\,\sqrt{n/(2\k^{1.5})}$ 
and $\|\cP_\Omega(E)\|_2 \leq \sigma_\k(M)/(16\sqrt{\k})$,  
we have $\|\Fnoise^{(t+1)}\|_2 \leq \sigma_\k(M)/8 $ 
and $\|\Fmiss^{(t+1)} \|_2 \leq (1/6) \sigma_\k(M) \, d(\hU^{(t)},U)$. 
Assuming $d(\hU^{(t)},U) \leq 1 / 2$, 
 this proves the desired claim. 
\end{proof}
\begin{proof}[Proof of Lemma \ref{lem:amincoherence}]
\label{sec:amincoherence}
Assuming that $\hU^{(t)}$ is $\tmu$-incoherent,  
we make use of the following set of inequalities: 
\begin{eqnarray*}
	\|(B^{(i,a)})^{-1}\|_2 &\geq& 1-(\tmu^2\k/n)\\
	\|B^{(i,a)}\|_2 &=& \|\id_{\k\times\k} - \hU_{(i)}\hU_{(i)}^T\|_2 \;\;\leq\;\; 1 \\
	\|D\|_2 &=&  \|\hU^T U\|_2 \;\;\leq \;\;1 \\
	\|C^{(i,a)}\|_2 &=& \| \hU^TU - \hU_{(i)}  U_{(i)}^T\|_2 \;\;\leq\;\; 1+\mu\tmu\,\k/n \;. 
\end{eqnarray*}
Also, from Lemma \ref{lem:amQR}, we know that 
if $\tmu\leq\sqrt{n/2\k}$ as per our assumption, 
then $\|(R_U^{(t+1)})^{-1}\|_2 \leq \sqrt{3}/\sigma_\k(\M)$.   
Then, by \eqref{eq:amupdaterule} and the triangular inequality, 
\begin{eqnarray*}
	 \sum_{a\in[\ell]} \| \hU^{(t+1)}_{(i,a)} \|^2 
		&\leq&  \sum_{a\in[\ell]} \big\|  (B^{(i,a)})^{-1}\, C^{(i,a)}\,S  \,U_{(i,a)} \,+\, (B^{(i,a)})^{-1}\, N_{(i,a)}\big\|^2  \,\big\|\big( R_U^{(t+1)}\big)^{-1}\big\|_2^2 \\ 
		&\leq& \sum_{a\in[\ell]} 2\big\|\big( R_U^{(t+1)}\big)^{-1}\big\|_2^2  \, \big\|(B^{(i,a)})^{-1}\big\|_2^2\, 
			\Big\{   \, \big\|C^{(i,a)} \big\|_2^2 \,\|S\|_2^2  \,\|U_{(i,a)}\|^2 \,+\,  \|N_{(i,a)}\|^2 \Big\} \\ 
		&\leq& \frac{6}{\sigma_\k(M)^2\, (1-(\tmu^2\,\k/n))^2} \, \sum_{a\in[\ell]}  \Big\{ \sigma_1(\M)^2  \Big( 1+\mu\tmu \k/n \Big)\|U_{(i,a)}\|^2\,   +\,  \big\| \hU ^T\,\cP_\Omega(E)_{(i,a)} \big\|^2  \Big\}  \\
		&\leq& \frac{6}{\sigma_\k(M)^2\, (1-(\tmu^2\,\k/n))^2} \,  \Big\{ \sigma_1(\M)^2\Big(1+\frac{\mu\tmu \k}{n}\Big) \frac{\mu^2 \k}{n} \,+\, \| \cP_\Omega(E)_{(i,a)} \|^2  \Big\} \\
		&\leq& \frac{36 \sigma_1(M)^2}{\sigma_\k(M)^2 } \,  \frac{\mu^2 \k}{n}  \;, 
\end{eqnarray*}
where the last inequality follows from 
our assumption that $\tmu\leq \sqrt{n/(2\k)}$, 
$\mu\leq (\sigma_\k(M)/\sigma_1(M))\sqrt{n/(32\k)} $, 
and $\|\cP_\Omega(E)_{(i,a)}\| \leq \sigma_1(M) \mu \,\sqrt{3\k/(8n\ell)}$.
This proves that 
$\hU^{(t+1)}$ is $\mu_1$-incoherent for 
$\mu_1 = 6 \mu (\sigma_1(\M)/\sigma_\k(\M))$. 
\end{proof}
\begin{proof}[Proof of Lemma \ref{lem:aminitial}]
\label{sec:aminitial}
Let $\cP_\k(\hM)=\tU \tS \tU^T$ denote the best rank-$\k$ approximation of 
the observed matrix $\hM$ and 
$\cP_\Omega$ is the sampling mask operator defined in \eqref{eq:defcP} 
such that $M-\hM = \cP_\Omega(E) + M-\cP_\Omega(M)$.  Then, 
\begin{eqnarray}
	 \|\, \M-\cP_\k(\hM) \,\|_2 
		&\leq& \|\, \M- \hM \,\|_2 + \|\, \hM-\cP_\k(\hM) \,\|_2  \nonumber\\
		&\leq& 2\, \|\, \M- \hM \,\|_2 \nonumber\\
		&\leq& 2\, \big(\|\cP_\Omega(E)\|_2 + \|\M- \cP_{\Omega}(\M)\|_2 \big) \nonumber\\
		&\leq& 2\, \big(\|\cP_\Omega(E)\|_2 + \sigma_1(\M)\,\musp^2\k/n \big) \;, \label{eq:aminitspectral} 
\end{eqnarray}
where we used the fact that 
$\cP_\k(\hM)$ is the best rank-$\k$ approximation such that 
$\| \hM-\cP_\k(\hM)\|_2 \leq \|\hM-A\|_2$ for any rank-$\k$ matrix $A$, 
and $\|\M- \cP_{\Omega}(\M)\|_2 = \max_i \|U_{(i)}SU_{(i)}^T\|_2 \leq (\musp^2\k/\n)\sigma_1(M)$. 

The next series of inequalities provide an upper bound on $d(\tU,U)$ in terms of the spectral norm: 
\begin{eqnarray*}
	 \|\, \M -\cP_\k(\hM)\,\|_2 &=&  \|\, (\tU\tU^T) (USU^T-\tU\tS\tU^T) + (\tU_\perp\tU_\perp^T)(USU^T-\tU\tS\tU^T)\,\|_2\\
	 	&\geq&  \|(\tU_\perp\tU_\perp^T)(USU^T-\tU\tS\tU^T)\,\|_2 \\
		&=& \|\tU_\perp^T USU^T \,\|_2\\
		&\geq& \sigma_\k(S)\, \|\tU_\perp^T U\|_2  \\
		&\geq& \sigma_\k(S)\, d(\tU,U) \;,
\end{eqnarray*}
Together with \eqref{eq:aminitspectral}, this implies that 
\begin{eqnarray*}
	d(\tU,U) &\leq& \frac{2}{\sigma_\k(\M)}\big(\|\cP_\Omega(E)\|_2 + \sigma_1(\M) \musp^2 \k/n \big)\;.
\end{eqnarray*}
For $\|\cP_\Omega(E)\|_2 \leq \sigma_\k(\M)/(32\,\sqrt{\k})$ and 
$\musp \leq \sqrt{ \sigma_\k(M)\,n/(32\,\sigma_1(M)\,\k^{1.5})}$ as per our assumptions, 
we have 
\begin{eqnarray*}
	d(\tU,U) &\leq& \frac{1}{8\,\sqrt{\k}} \;.
\end{eqnarray*} 

Next, we show that by truncating large components of $\tU$, 
we can get an incoherent matrix $\hU^{(0)}$ which is also close to $U$. 
Consider a sub-matrix of $U$ which consists of the rows 
from $\ell(i-1)+1$ to $\ell i$. 
We denote this block by $U_{(i)}\in\reals^{\ell\times\k}$. 
Let $\bU$ denote an $\ell\n\times \k$ matrix obtained from $\tU$ 
by setting to zero all blocks that have Frobenius norm 
greater than $2\mu\sqrt{\k/n}$. 
Let $\hU^{(0)}$ be the orthonormal basis of $\bU$. 
We use the following lemma to bound the error and incoherence 
of the resulting $\hU^{(0)}$. A similar lemma has been proven in \cite[Lemma C.2]{JNS13}, 
and we provide a tighter bound in the following lemma.
For $\delta\leq1/(8\sqrt{\k})$, this lemma proves that 
we get the desired bound of $d(\hU^{(0)},U)\leq1/2$ and $\mu_0\leq 4\mu$. 
\end{proof}
\begin{lemma}
	Let $\mu_0$ denote the incoherence of $\bU$, 
	and define $\delta\equiv d(\tU,U)$. Then 
	\begin{eqnarray*}
		d(\hU^{(0)},U) \;\leq\; \frac{3\sqrt{\k}\,\delta}{1-2\sqrt{\k}\,\delta}\;,\;\;\;\;\; \text{ and }\;\;\;\;\;
		\mu_0 \;\leq\; \frac{2\mu}{1-2\sqrt{\k}\,\delta}\;.
	\end{eqnarray*}
\end{lemma}
\begin{proof}
	Denote the QR decomposition of $\bU$ by $\bU=\hU^{(0)} R$
	and let $\delta\equiv d(\tU,U)$. 
	Then, 
	\begin{eqnarray}
		d(\hU^{(0)},U) &=& \| U_\perp^T \hU^{(0)} \|_2 \nonumber\\
			&\leq& \| U_\perp^T\bU\|_2\,\|R^{-1}\|_2\nonumber\\
			&\leq& \big( \| U_\perp^T(\bU-\tU)\|_2 + \| U_\perp^T \tU \|_2 \big)\, \|R^{-1}\|_2\nonumber\\
			&=& \big( \| \bU-\tU\|_2 + \delta \big)\, \|R^{-1}\|_2\;.\label{eq:aminitbound}
	\end{eqnarray}
	
	First, we upper bound $\|\bU-\tU\|_F$ as follows. 
	Let $\cP()$ denote a projection operator that 
	sets to zero those blocks whose Frobenius norm is smaller than $2\mu\sqrt{\k/n}$ 
	such that $\cP(\tU)=\tU-\bU$. Then, 
	\begin{eqnarray}
		\|\cP(\tU)\|_F &\leq& \|\cP(\tU-U(U^T\tU)) \|_F \,+\, \|\cP(U(U^T\tU))\|_F \;. \label{eq:aminittU}
	\end{eqnarray}
	The first term can be bounded by 
	$ \|\cP(\tU-U(U^T\tU)) \|_F \leq \|\tU-U(U^T\tU) \|_F \leq \sqrt{\k} \|\tU-U(U^T\tU) \|_2 = \sqrt{\k}\, \delta$. 
	The second term can be bounded by 
	$ \|\cP(U(U^T\tU))\|_F = \|\cP(U)\,(U^T\tU)\|_F \leq  \| \cP(U)\|_F$. 
	 By incoherence of $U$, we have that 
	$ \| \cP(U)\|_F  \leq \sqrt{N} \mu\sqrt{\k/n}$, 
	where $N$ is the number of $\ell\times\k$ block matrices that is not set to zero by $\cP(\cdot)$. 
	
	To provide an upper bound on $N$, notice that 
	the incoherence of an $\ell\n\times\k$ matrix $U(U^T\tU)$ is $\mu$. 
	This follows from the fact that $\|U^T\tU\|_2\leq 1$. 
	Then, 
	\begin{eqnarray*} 
		\|U(U^T\tU)-\tU\|_F &\geq& \|\cP(U(U^T\tU)-\tU) \|_F\\
			&\geq& \sqrt{N}\, \mu\sqrt{\frac{\k}{n}}\;,
	\end{eqnarray*}
	where the last line follows from the fact that there are $N$ blocks where 
	the Frobenius norm of $U(U^T\tU$ in that block is at most $\mu\sqrt{\k/n}$ and 
	the Frobenius norm of $\tU$ is at least $2\mu\sqrt{\k/n}$. 
	On the other hand, we have $\|U(U^T\tU)-\tU\|_F \leq \sqrt{\k} \delta$. 
	Putting these inequalities together, we get that 
	\begin{eqnarray*}
		\sqrt{N} \leq \frac{\delta\,\sqrt{n}}{\mu} \;\; \text{ and } \|\cP(U(U^T\tU))\|_F \leq \sqrt{\k}\,\delta\;.
	\end{eqnarray*}
	Substituting these bounds in \eqref{eq:aminittU} gives 
	\begin{eqnarray}
		\| \tU-\bU \|_F &\leq&  2\,\delta\,\sqrt{\k}\;. \label{eq:aminitbound1}
	\end{eqnarray}
	Next, we show that 
	\begin{eqnarray}
		\| R^{-1}\|_2 &\leq&  \frac{1}{1-2\delta\sqrt{\k}}\;. \label{eq:aminitbound2}
	\end{eqnarray}
	By the definition of $R$, 
	we know that $\|R^{-1}\|_2 = 1/\sigma_\k(R) = 1/\sigma_\k(\hU^{(0)}=1/\sigma_\k(\bU)$. 
	Using Weyl's inequality, 
	we can lower bound $\sigma_\k(\bU) = \sigma_\k(\bU-\tU+\tU) \geq \sigma_\k(\tU)-\sigma_1(\bU-\tU)$.
	Since $\tU$ is an orthogonal matrix and using \eqref{eq:aminitbound1}, 
	this proves \eqref{eq:aminitbound2}. 
	Substituting \eqref{eq:aminitbound1} and 
	\eqref{eq:aminitbound2} into \eqref{eq:aminitbound}, 
	we get 
	\begin{eqnarray*}
		d(\hU^{(0)},U) &\leq& \frac{(2\sqrt{\k}+1)\delta}{1-2\delta\sqrt{\k}}\;.
	\end{eqnarray*}
	For $\delta\leq$, this gives the desired bound. 
	
	To provide an upper  bound on the incoherence $\mu_0$ of $\hU^{(0)}$, 
	recall that the incoherence is defined as  
	$\mu_0 \sqrt{\k/n} = \max_i \|\hU^{(0)}_{(i)}\|_F = 	\max_i \|\bU_{(i)}R^{-1}\|_F$. 
	By construction, $\|\bU_{(i)}\|_F \leq 2\mu\sqrt{\k/n}$, 
	and from \eqref{eq:aminitbound2} we know that $\|R^{-1}\|_2\leq 1/(1-2\delta\sqrt{\k})$. 
	Together, this gives 
	\begin{eqnarray*}
		\mu_0 &\leq& \frac{2\mu}{1-2\delta\sqrt{\k}}\;.
	\end{eqnarray*}
	This finishes the proof of the desired bounds. 
\end{proof}

%%% Local Variables: 
%%% mode: latex
%%% TeX-master: "crowd"
%%% End: 

\subsection{Proof of Theorem~\ref{thm:tensorls}}
In this section, we provide a detailed proof of Theorem~\ref{thm:tensorls}. To this end, we first provide an infinite sample version of the proof, i.e., when $P_{\Omega_3}(S_3)=P_{\Omega_3}(M_3)$. Then, in the next subsection, we bound each element of $P_{\Omega_3}(S_3)-P_{\Omega_3}(M_3)$ and extend the infinite sample version of the proof to the finite sample case.

Recall that $\hM_2=\hU_{M_2} \hSigma_{M_2} \hU_{M_2}^T$, 
$\varepsilon = \|\hM_2-M_2\|_2\|/ \sigma_r(\M_2)$, 
$\M_2$ is $\mu$-incoherent and $\hM_2$ is $\mu_1$-incoherent. 
Incoherence of a matrix is defined as in \eqref{eq:defincoherence}. Then, the following two remarks can be easily proved using standard matrix perturbation results (for example, see~\cite{AHK12}). 
\begin{remark}\label{claim:tls_cl1}
Suppose $\|\hM_2-M_2\|_2\leq \varepsilon \sigma_r(M_2)$, then
$$1-4\frac{\varepsilon^2}{(1-\varepsilon)^2}\leq \sigma_r(U^T\hum)\leq \sigma_1(\hum^T U) \leq 1.$$
That is, 
$$\|(I-\hU_{M_2}\hU_{M_2}^T)U\|_2\leq \varepsilon,\text{ and },$$
$$\|(U^T\hum)^T(U^T\hum)-I\|\leq 8\frac{\varepsilon^2}{(1-\varepsilon)^2}.$$
\end{remark}

\begin{remark}\label{claim:tls_cl2}
  Suppose $\|\hM_2-M_2\|_2\leq \varepsilon\sigma_r(M_2)$, then 
$$\|\id-\hSigma_{M_2}^{-1/2}\hU_{M_2}^T M_2 \hU_{M_2}\hSigma_{M_2}^{-1/2}\|_2\leq 2\varepsilon.$$%\frac{1}{w_{min}\sigma_r^2-\varepsilon}\varepsilon^2. $$
\end{remark}
\begin{proof}

\begin{eqnarray*}
	\|I-\hSigma_{M_2}^{-1/2}\hU_{M_2}^T M_2 \hU_{M_2}\hSigma_{M_2}^{-1/2}\|_2 &=& \|\hSigma_{M_2}^{-1/2}\hU_{M_2}^T (\hM_2- M_2) \hU_{M_2}\hSigma_{M_2}^{-1/2}\|_2\\
		&\leq&   \|\hSigma_{M_2}^{-1/2}\hU_{M_2}^T\|_2^2 \, \|\hM_2- M_2\|_2 \\
		&\leq& \frac{1}{\sigma_r(\M_2)(1-\varepsilon)}\sigma_r(\M_2) \varepsilon \;,
\end{eqnarray*}
where we used the fact that 
$\|\hSigma_{M_2}^{-1/2}\|_2^2  \geq 1/\sigma_r(\hM_2)$ 
and $\sigma_r(\hM_2) \geq \sigma_r(\M_2) (1-\varepsilon)$ by Weyl's inequality.
For $\varepsilon<1/2$ we have the desired bound.  
\end{proof}

\noindent We now define the following operators: $\hnu$ and $\hA$. 
%Define $\nu:\R^{r\times r\times r}\rightarrow \R^{\ell n\times \ell n\times \ell n}$ as: 
%\begin{equation}
%  \label{eq:tls_nu}
%  \nu_{ijk}(Z)=\begin{cases}\sum_{abc}Z_{abc}(UU^T\hum\hsm^{1/2})_{ia}(UU^T\hum\hsm^{1/2})_{jb}(UU^T\hum\hsm^{1/2})_{kc},& \text{ if } \lceil \frac{i}{\ell} \rceil\neq \lceil \frac{j}{\ell} \rceil\neq \lceil \frac{k}{\ell} \rceil\neq \lceil \frac{i}{\ell} \rceil, \\
%0, & \mbox{otherwise}.\end{cases}
%\end{equation}
Define $\hnu:\R^{r\times r\times r}\rightarrow \R^{\ell n\times \ell n\times \ell n}$ as: 
\begin{equation}
  \label{eq:tls_hnu}
  \hnu_{ijk}(Z)=\begin{cases}\sum_{abc}Z_{abc}(\hU_{M_2}\hsm^{1/2})_{ia}(\hU_{M_2}\hsm^{1/2})_{jb}(\hU_{M_2}\hsm^{1/2})_{kc},& \text{ if } \lceil \frac{i}{\ell} \rceil\neq \lceil \frac{j}{\ell} \rceil\neq \lceil \frac{k}{\ell} \rceil\neq \lceil \frac{i}{\ell} \rceil, \\
0, & \mbox{otherwise}.\end{cases}
\end{equation}
%Now, define $A: \R^{\ell n\times \ell n\times \ell n}\rightarrow \R^{r\times r\times r}$ as: \begin{equation}\label{eq:tls_a}A(Z)=\nu(Z)\left[\hU_{M_2}\hsm^{-1/2},\hU_{M_2}\hsm^{-1/2},\hU_{M_2}\hsm^{-1/2}\right].\end{equation}
Define $\hA: \R^{\ell n\times \ell n\times \ell n}\rightarrow \R^{r\times r\times r}$ as: 
\begin{equation}
\label{eq:tls_ha}
\hA(Z)=\hnu(Z)\left[\hU_{M_2}\hsm^{-1/2},\hU_{M_2}\hsm^{-1/2},\hU_{M_2}\hsm^{-1/2}\right].\end{equation}
%Also, define the following orthogonal matrix: 
%\begin{equation}
%  \label{eq:tls_r}
%  R_3=W^{-1/2}V\Sigma^{-1}U^T\um\sm^{1/2}.
%\end{equation}
%Also, note that $\um\sm^{1/2}=U\Sigma V^T W^{1/2} R_3$. 
Now, let $R_3$ be defined as: $R_3=\hsm^{-1/2}\hum^T U \Sigma V^T W^{1/2}$. Note that, using Remark~\ref{claim:tls_cl2}, 
$$\|R_3R_3^T-I\|\leq 2\varepsilon.$$
Also, define the following tensor:
\begin{equation}
	\label{eq:tls_g}
		\tG=\sum_{q\in [r]} \frac{1}{\sqrt{w_q}} (R_3\e_q \otimes R_3\e_q \otimes R_3\e_q).
\end{equation}
Note that, as $R_3$ is nearly orthonormal, $\tG$ is a {\em nearly} orthogonally decomposable tensor. 

We now present a lemma that shows that $P_{\Omega_3}(M_3)\left[\hU_{M_2}\hsm^{-1/2},\hU_{M_2}\hsm^{-1/2},\hU_{M_2}\hsm^{-1/2}\right]$ and $\hA(\tG)$ are ``close''. 
%Now, Claim~\ref{claim:tls_haz} shows that: 
%$$P_{\Omega_3}(M_3)\left[\hU_{M_2}\hsm^{-1/2},\hU_{M_2}\hsm^{-1/2},\hU_{M_2}\hsm^{-1/2}\right]=\hA(G)+E,$$
%where $\|E\|_F\leq $. 

%Now, define $G$ as:
%Note that, $\nu_{ijk}\left(G\right)=P_{\Omega_3}(M_3)$. That is, 
%\begin{equation}
%  \label{eq:tls_3}
%  A\left(G\right)=P_{\Omega_3}(M_3)\left[\hU_{M_2}\hsm^{-1/2},\hU_{M_2}\hsm^{-1/2},\hU_{M_2}\hsm^{-1/2}\right].
%\end{equation}
\begin{lemma}\label{lemma:tls_haz}
  $$P_{\Omega_3}(M_3)\left[\hU_{M_2}\hsm^{-1/2},\hU_{M_2}\hsm^{-1/2},\hU_{M_2}\hsm^{-1/2}\right]=\hA(\tG)+E,$$
where 
\begin{eqnarray*}
	\|E\|_F &\leq&  \frac{12 \,\mu_1^3 \, \mu \, r^{3.5} \, \sigma_1(M_2)^{3/2}\, \varepsilon}{n\sqrt{\wmin}  \sigma_r(M_2)^{3/2}}  \;,
\end{eqnarray*} 
and we denote the Frobenius norm of a tensor as $\|E\|_F = \{\sum_{i,j,k}E_{i,j,k}^2\}^{1/2}$
\end{lemma}
\begin{proof}
Define $H=\hA(G)$ and $F=P_{\Omega_3}(M_3)\left[\hU_{M_2}\hsm^{-1/2},\hU_{M_2}\hsm^{-1/2},\hU_{M_2}\hsm^{-1/2}\right]$. Also, let $Q=U\Sigma V^TW^{1/2}$ and $\hQ=\hU_{M_2}\hsm^{-1/2}$. 

Note that, $F_{abc}=\sum_{ijk}\delta_{ijk}M_3(i,j,k)\hQ_{ia}\hQ_{jb}\hQ_{kc}, $ where $\delta_{ijk}=1,\text{ if }(i,j,k)\in \Omega_3$ and $0$ otherwise. 
 Also, $M_3(i,j,k)=\sum_{q\in [r]}\frac{1}{\sqrt{w_q}}Q_{iq}\cdot Q_{jq}\cdot Q_{kq}.$ Hence, \begin{equation}F_{abc}=\sum_{q\in [r]}\frac{1}{\sqrt{w_q}}\sum_{ijk}\delta_{ijk}Q_{iq}\cdot Q_{jq}\cdot Q_{kq}\cdot \hQ_{ia}\cdot \hQ_{jb}\cdot \hQ_{kc}.\end{equation}
Note that, $\sum_{i}\hQ_{ia}Q_{iq}=\ip{\hQ_{a}}{Q_{q}}=\e_a^T\hsm^{-1/2}\hU_{M_2}^TU\Sigma V^TW^{1/2}\e_q=\e_a^TR_3\e_q$. 
That is, 
\begin{multline}
  F_{abc}=G_{abc}-\sum_{q\in [r]}\frac{1}{\sqrt{w_q}}\sum_{m\in [n]} \ip{\hQ^{(m)}_a}{Q^{(m)}_q}\cdot \ip{\hQ^{(m)}_b}{Q^{(m)}_q} \cdot \ip{\hQ^{(m)}_c}{Q^{(m)}_q}\\-\sum_{q\in [r]}\frac{1}{\sqrt{w_q}}\e_a^TR_3\e_q\sum_{m\in [n]}  \ip{\hQ^{(m)}_b}{Q^{(m)}_q} \cdot \ip{\hQ^{(m)}_c}{Q^{(m)}_q}-\sum_{q\in [r]}\frac{1}{\sqrt{w_q}}\e_b^TR_3\e_q\sum_{m\in [n]}  \ip{\hQ^{(m)}_a}{Q^{(m)}_q} \cdot \ip{\hQ^{(m)}_c}{Q^{(m)}_q}\\-\sum_{q\in [r]}\frac{1}{\sqrt{w_q}}\e_c^TR_3\e_q\sum_{m\in [n]}  \ip{\hQ^{(m)}_a}{Q^{(m)}_q} \cdot \ip{\hQ^{(m)}_b}{Q^{(m)}_q}.
\end{multline}
On the other hand, 
\begin{equation}\hnu(G)_{ijk}=\begin{cases}\sum_{q\in [r]}\frac{1}{\sqrt{w_q}} \e_i^T(\hU_{M_2}\hsm^{1/2}R_3)\e_q\cdot \e_j^T(\hU_{M_2}\hsm^{1/2}R_3)\e_q\cdot \e_k^T(\hU_{M_2}\hsm^{1/2}R_3)\e_q,& \text{ if } \lceil \frac{i}{\ell} \rceil\neq \lceil \frac{j}{\ell} \rceil\neq \lceil \frac{k}{\ell} \rceil\neq \lceil \frac{i}{\ell} \rceil, \\
0, & \mbox{otherwise}.\end{cases}\end{equation}
That is, 
\begin{equation}
  H_{abc}=\sum_{q\in [r]}\frac{1}{\sqrt{w_q}} \sum_{ijk}\delta_{ijk}\e_i^T(\hU_{M_2}\hsm^{1/2}R_3)\e_q\cdot \e_j^T(\hU_{M_2}\hsm^{1/2}R_3)\e_q\cdot \e_k^T(\hU_{M_2}\hsm^{1/2}R_3)\e_q\cdot \hQ_{ia}\cdot \hQ_{jb}\cdot \hQ_{kc}. 
\end{equation}
Now,  note that $\sum_{i}\hQ_{ia}\e_i^T(\hU_{M_2}\hsm^{1/2}R_3)\e_q=\ip{\hQ_{a}}{\hU_{M_2}\hsm^{1/2}R_3\e_q}=\e_a^T\hsm^{-1/2}\hU_{M_2}^T\hU_{M_2}\hsm^{1/2}R_3\e_q=\e_a^TR_3\e_q$. Also, let $\Uc=\hum\hum^TQ$. That is, 
\begin{multline}
  H_{abc}=G_{abc}-\sum_{q\in [r]}\frac{1}{\sqrt{w_q}}\sum_{m\in [n]} \ip{\hQ^{(m)}_a}{\Uc^{(m)}_q}\cdot \ip{\hQ^{(m)}_b}{\Uc^{(m)}_q} \cdot \ip{\hQ^{(m)}_c}{\Uc^{(m)}_q}\\-\sum_{q\in [r]}\frac{1}{\sqrt{w_q}}\e_a^TR_3\e_q\sum_{m\in [n]}  \ip{\hQ^{(m)}_b}{\Uc^{(m)}_q} \cdot \ip{\hQ^{(m)}_c}{\Uc^{(m)}_q}-\sum_{q\in [r]}\frac{1}{\sqrt{w_q}}\e_b^TR_3\e_q\sum_{m\in [n]}  \ip{\hQ^{(m)}_a}{\Uc^{(m)}_q} \cdot \ip{\hQ^{(m)}_c}{\Uc^{(m)}_q}\\-\sum_{q\in [r]}\frac{1}{\sqrt{w_q}}\e_c^TR_3\e_q\sum_{m\in [n]}  \ip{\hQ^{(m)}_a}{\Uc^{(m)}_q} \cdot \ip{\hQ^{(m)}_b}{\Uc^{(m)}_q}.
\end{multline}
Now, 
\begin{eqnarray*}
	\Big| \ip{\hQ^{(m)}_c}{\Uc^{(m)}_q}-\ip{\hQ^{(m)}_c}{Q^{(m)}_q} \Big| &\leq& \|\hQ^{(m)}_c\|\, \|\Uc^{(m)}_q-Q^{(m)}_q\| \\
	&\leq& \|\hQ^{(m)}_c\|\, \|(I-\hum\hum^T)U\|_2\, \|\Sigma V^TW^{1/2}\|_2 \\
	&\leq& \frac{\mu_1 \sqrt{r}}{\sqrt{n (1-\varepsilon) \sigma_r(\M_2)}} \,  \varepsilon \,\sqrt{\sigma_1(\M_2) }  , 
\end{eqnarray*} 
where we used $ \|(I-\hum\hum^T)U\|_2 \leq\varepsilon$ from Remark \ref{claim:tls_cl1}, 
and the following remark to bound $\|\hQ^{(m)}_c\|$. 
Then, from Remark \ref{rem:tls_cl3}, 
\begin{align*}
	& | \<\hQ^{(m)}_a,\tQ^{(m)}_q \> \<\hQ^{(m)}_c,\tQ^{(m)}_q \> \,-\, \<\hQ^{(m)}_a, Q^{(m)}_q \> \<\hQ^{(m)}_c, Q^{(m)}_q \>  | \\
	& \;\; \leq  \;\; |( \<\hQ^{(m)}_a,\tQ^{(m)}_q \> \,-\, \<\hQ^{(m)}_a, Q^{(m)}_q \> ) \<\hQ^{(m)}_c,\tQ^{(m)}_q \>  | + |  \<\hQ^{(m)}_a, Q^{(m)}_q \>  (\<\hQ^{(m)}_c,\tQ^{(m)}_q \> \,-\,  \<\hQ^{(m)}_c, Q^{(m)}_q \>)  | \\
	&\;\;\leq \;\; \frac{\mu_1 \sqrt{r \,\sigma_1(\M_2)}}{\sqrt{n (1-\varepsilon) \sigma_r(\M_2)}} \,  \varepsilon \,  \frac{\mu_1(\mu+\mu_1)\, r}{n\,(1-\varepsilon)\,\sigma_r(M_2)}
\end{align*}
Further, $|e_a^TR_3e_q| \leq \mu_1\sqrt{(r\sigma_1(M_2))/(n(1-\varepsilon)\sigma_r(M_2))} $. 
The desired bound now follows by using the above inequalities to bound $\|E\|_F = \|H-F\|_F$.
\end{proof}
%============================================================
\begin{remark}
	\label{rem:tls_cl3}
	For $\tQ=\hU_{\M_2}\hU_{\M_2}^T Q$, $Q=U\Sigma V^TW^{1/2}$, and $\hQ=\hU_{M_2}\hsm^{-1/2}$,   
	suppose $\M_2$ is $\mu$-incoherent and $\hM_2$ is $\mu_1$-incoherent. Then, 
	$$\|\hQ^{(m)}_c\| \leq  \frac{\mu_1 r^{1/2}}{\sqrt{ (1 - \varepsilon)  \,n\,\sigma_r(\M_2)}}\;, 
	\;\;\;\;  \|\tQ^{(m)}_c \| \leq  \mu_1  \sqrt{\frac{r \sigma_1(\M_2)}{n}} 
	\;, \;\;\text{ and }\;\;\| Q^{(m)}_c \| \leq  \mu \sqrt{\frac{r \sigma_1(\M_2)}{n}}\;. $$
\end{remark} 
\begin{proof}
	\begin{eqnarray*}
		\|\hQ^{(m)}_c\| &=& \frac{1}{\sqrt{\hSigma_{cc}}}  \Big\{\sum_{a\in[\ell]} (\hU_{M_2})_{\ell (m-1) + a,c}   \Big\}^{1/2} \; \leq \frac{ \mu_1 \sqrt{r/n}}{\sqrt{\sigma_r(\M_2)(1-\varepsilon)}} \; .  
	\end{eqnarray*}
	The rest of the remark follows similarly.
\end{proof} 
%Hence, the only difference between terms of $F_{abc}$ and $H_{abc}$ is that %$F_{abc}$ has terms of type $\e_i^TU\Sigma V^TW^{1/2}\e_q$ while $H_{abc}$ has terms of type $\e_i^T\hU_{M_2}\hsm^{1/2}R_3\e_q=\e_i\hU_{M_2}\hU_{M_2}^TU\Sigma V^TW^{1/2}\e_q $. 
%Now, note that $\|\hU_{M_2}\hU_{M_2}^TU-U\|_2=\|(I-\hU_{M_2}\hU_{M_2}^T)U\|_2\leq \varepsilon/\sigma_r(M_2)$. Hence, $\|E\|_F$ decreases as $\varepsilon$ decreases. 
%\begin{claim}
%  $\|A-\hA\|_2\leq $
%\end{claim}
%\begin{proof}
%  The proof follows by using Claim~\ref{claim:tls_cl1} and the incoherence of $\hum, \um$. 
%\end{proof}

%============================================================
Next, we now  show that $\|\hA^{-1}\|_2$  is small. 
\begin{lemma}\label{lemma:tls_ha}
  $\sigma_{\rm min}(\hA)\geq 1-8r^3\sigma_1(M_2)^2 (1+\varepsilon)^2/(n\sigma_r(M_2)^2(1-\varepsilon)^2)$ 
  and hence, 
  $$\|\hA^{-1}\|_2 \;\leq\; \frac{1}{1-72r^3\sigma_1(M_2)^2 /(n\sigma_r(M_2)^2)}\;.$$ 
\end{lemma}
\begin{proof}
Let $\hQ=\hU_{M_2}\hsm^{-1/2}$, $\Uc=\hU_{M_2}\hsm^{1/2}$, and $H=\hA(Z)$. Then, $$H_{abc}=\sum_{ijk}\delta_{ijk} \sum_{a'b'c'}Z_{a'b'c'} \Uc_{ia'}\cdot \Uc_{jb'}\cdot \Uc_{kc'}\cdot \hQ_{ia}\cdot \hQ_{jb}\cdot \hQ_{kc},$$
where $\delta_{ijk}=1,\text{ if }(i,j,k)\in \Omega_3$ and $0$ otherwise. 
That is, 
\begin{multline}
  H_{abc}=Z_{abc}-\sum_{a'b'c'}Z_{a'b'c'}\sum_{m\in [n]}\ip{\hQ^{(m)}_a}{\Uc^{(m)}_{a'}}\cdot \ip{\hQ^{(m)}_b}{\Uc^{(m)}_{b'}} \cdot \ip{\hQ^{(m)}_c}{\Uc^{(m)}_{c'}}\\-\sum_{b'c'}Z_{ab'c'}\sum_{m\in [n]}  \ip{\hQ^{(m)}_b}{\Uc^{(m)}_{b'}} \cdot \ip{\hQ^{(m)}_c}{\Uc^{(m)}_{c'}}-\sum_{a'c'}Z_{a'bc'}\sum_{m\in [n]}  \ip{\hQ^{(m)}_a}{\Uc^{(m)}_{a'}} \cdot \ip{\hQ^{(m)}_c}{\Uc^{(m)}_{c'}}\\-\sum_{a'b'}Z_{a'b'c}\sum_{m\in [n]}  \ip{\hQ^{(m)}_{a}}{\Uc^{(m)}_{a'}} \cdot \ip{\hQ^{(m)}_c}{\Uc^{(m)}_{c'}}.
\end{multline}
Let ${\rm vec}(H)=B\cdot {\rm vec}(Z)$. 
We know that $|\< \tQ^{(m)}_a,\hQ^{(m)}_a\>| \leq \mu_1^2 \k/\n $ and 
$ |\<\tQ^{(m)}_a,\tQ^{(m)}_{a'}\>| \leq \mu_1^2  \k\sigma_1(\M_2)(1+\varepsilon) /(\n\sigma_r(\m_2)(1-\varepsilon)) $ for $a\neq a'$.
Now, using the above equation and using incoherence: 
$$1-4r^2 \mu_1^4 /\n \leq B_{pp}\leq 1+4r^2 \mu_1^4 /\n, \forall 1\leq p\leq r.$$
Similarly, 
$|B_{pq}| \leq 4r^2 \mu_1^4 \sigma_1(\M_2)^2(1+\varepsilon)^2 /(\n \sigma_r(M_2)^2(1-\varepsilon)^2), \forall p\neq q$. 
Theorem now follows using Gershgorin's theorem. 
\end{proof}

%============================================================
Finally, we combine the above two lemmas to show that the least squares procedure approximately recovers $\tG$. 
\begin{lemma}\label{lemma:tls_main}
	Let $G$ be as defined in \eqref{eq:tls_g}. Also, let $\hG$ be obtained by solving the following least squares problem:
	$$\hG=\arg\min_{Z}\|\hA(Z)-P_{\Omega_3}(M_3)\left[\hU_{M_2}\hsm^{-1/2},\hU_{M_2}\hsm^{-1/2},\hU_{M_2}\hsm^{-1/2}\right]\|_F^2.$$
	Then, for $n\geq 144 r^3 \sigma_1(M_2)^2/\sigma_r(M_2)^2$ such that $\|\hA^{-1}\|_2 \leq 2$, 
	\begin{eqnarray*}
		\|\hG-\tG\|_F &\leq& \frac{24 \mu_1^3 \, \mu \, r^{3.5} \sigma_1(M_2)^{3/2} \varepsilon }{n \sqrt{\wmin} \sigma_r(M_2)}  \;. 
	\end{eqnarray*} 
%Moreover, $\|\Pi-\hPi\|_2\leq $. 
\end{lemma}
\begin{proof}
Note that $\hA: \R^{r\times r\times r} \rightarrow \R^{r\times r\times r}$ is a square operator. Moreover, using Lemma~\ref{lemma:tls_haz}: 
$$P_{\Omega_3}(M_3)\left[\hU_{M_2}\hsm^{-1/2},\hU_{M_2}\hsm^{-1/2},\hU_{M_2}\hsm^{-1/2}\right]=\hA(\tG)+E.$$ 
Hence, $\|\hG-\tG\|_F = \|\hA^{-1}(\hA(\hG)-\hA(\tG))\|_2 \leq  \|\hA^{-1}\|_2 \,\|E\|_F\,$.  
Together with  Lemma~\ref{lemma:tls_haz} and \ref{lemma:tls_ha}, we get the desired bound.  
\end{proof}

%========================================
% \begin{lemma}\label{lemma:tls_main}
% Let $G$ be as defined in \eqref{eq:tls_g}. Let 
% $S_3=\frac{1}{|\S|}\sum_{t\in \cS } \sum_{(a,b,c)\in \Omega_3} x_{t,a}x_{t,b} x_{t,c}$. Let $\hG$ be obtained by solving the following least squares problem:
% $$\hG=\arg\min_{Z}\|\hA(Z)-P_{\Omega_3}(S_3)\left[\hU_{M_2}\hsm^{-1/2},\hU_{M_2}\hsm^{-1/2},\hU_{M_2}\hsm^{-1/2}\right]\|_F^2.$$
% Then, for $n\geq 144 r^3 \sigma_1(M_2)^2/\sigma_r(M_2)^2$ such that $\|\hA^{-1}\|_2 \leq 2$,
% \begin{eqnarray*}
% 	\|\hG-G\|_F &\leq & \Big(\,  \frac{12\mu_1^3\mu r^{3.5} \sigma_1(M_2)^{3/2} \varepsilon}{n\sqrt{\wmin}\sigma_r(M_2)^{3/2}} + \frac{6r^3\mu_1^3}{\sigma_r(M_2)^{3/2}}\sqrt{\frac{\log(1/\delta)}{|\cS|}}    \,\Big) \;.
% \end{eqnarray*}
% %Moreover, $\|\Pi-\hPi\|_2\leq $. 
% \end{lemma}
\begin{proof}[Proof of Theorem~\ref{thm:tensorls}]
Note that $A: \R^{r\times r\times r} \rightarrow \R^{r\times r\times r}$ is a square operator. Moreover, using Lemma~\ref{lemma:tls_haz}: 
$$P_{\Omega_3}(M_3)\left[\hU_{M_2}\hsm^{-1/2},\hU_{M_2}\hsm^{-1/2},\hU_{M_2}\hsm^{-1/2}\right]=\hA(\tG)+E.$$ 

In the case of finite many samples, we use $S_3=\frac{1}{|\S|}\sum_{t= 1+|\S|/2}^{|\S|} x_t \otimes x_t \otimes x_t$ for estimating the low-dimensional tensor $\tG$. In particular, we compute the following quantity: \begin{equation}\label{eq:fin_hh}\hH=P_{\Omega_3}(S_3)\left[\hU_{M_2}\hsm^{-1/2},\hU_{M_2}\hsm^{-1/2},\hU_{M_2}\hsm^{-1/2}\right].\end{equation} We then use this quantity to solve the least squares problem. That is, we find $\hG$ as: 
$$\hG=\arg\min_{Z}\|\hA(Z)-\hH\|_F^2.$$
Now, we show that such a procedure gives $\hG$ that is close to $\tG$ (see \eqref{eq:tls_g}). 

\begin{eqnarray*}
	\|\hG-\tG\|_F &=& \| \hA^{-1}(\hA(\hG)) - \hA^{-1}(\hA(\tG)) \|_F \\
		&=& \| \hA^{-1}(\cP_{\Omega_3}(S_3)[\hU_{M_2}\hsm^{-1/2},\hU_{M_2}\hsm^{-1/2},\hU_{M_2}\hsm^{-1/2}]) - \hA^{-1}(\hA(\tG)) \|_F \\
		&=& \| \hA^{-1}(\cP_{\Omega_3}(S_3-M_3+M_3)[\hU_{M_2}\hsm^{-1/2},\hU_{M_2}\hsm^{-1/2},\hU_{M_2}\hsm^{-1/2}]) - \hA^{-1}(\hA(\tG)) \|_F \\
		&\leq & \| \hA^{-1}\|_2 \, \Big( \,  \|E \|_F  + \|\cP_{\Omega_3}(S_3-M_3)[\hU_{M_2}\hsm^{-1/2},\hU_{M_2}\hsm^{-1/2},\hU_{M_2}\hsm^{-1/2}] \, \|_F \,\Big)   \\
	&\leq& \|A^{-1}\|_2 \,\Big(\,  \frac{12\mu_1^3\mu r^{3.5} \sigma_1(M_2)^{3/2} \varepsilon}{n\sqrt{\wmin}\sigma_r(M_2)^{3/2}} + 
	\|\cP_{\Omega_3}(S_3-M_3)[\hU_{M_2}\hsm^{-1/2},\hU_{M_2}\hsm^{-1/2},\hU_{M_2}\hsm^{-1/2}] \, \|_F \,\Big)
	%\frac{6r^3\mu_1^3}{\sigma_r(M_2)^{3/2}}\sqrt{\frac{\log(1/\delta)}{|\cS|}}    \,\Big)
	\;.
\end{eqnarray*}
This finishes the proof of the desired claim.
%where we used  Lemma~\ref{lemma:tls_sample} to get 
%$$\|\cP_{\Omega_2}(S_3-M_3)[\hU_{M_2}\hsm^{-1/2},\hU_{M_2}\hsm^{-1/2},\hU_{M_2}\hsm^{-1/2}]) \|_F \leq %(6r^3\mu_1^3/\sigma_r(M_2)^{3/2})\sqrt{\log(1/\delta)}{|\cS|}.$$
\end{proof}

%%% Local Variables: 
%%% mode: latex
%%% TeX-master: "crowd"
%%% End: 

%
%=========================================================================
%
\subsection{Proof of Lemma \ref{lem:conc_mx}}
\label{sec:conc_mx}

Let $E = E^{(1)} - E^{(2)}$ where 
$E^{(1)} \equiv S_2-\E[S_2]$,  
$E^{(2)} \equiv \cP_{\Omega_2^c}(S_2-\E[S_2]) $, and $\Omega_2^c$ is the complement of $\Omega_2$. 
We first note that $\|x_t\|^2=n$. Hence, applying Matrix Hoeffding bound (see Theorem 1.3 of \cite{Tropp12}), we get with probability at least $1-\delta$: 
$$\|E^{(1)}\|_2 \;\;=\;\; \Big\|\frac{2}{|\S|}\sum_{t\in \{1,\ldots,|\cS|/2\}}(x_tx_t^T)-\E\Big[\frac{2}{|\S|}\sum_{t\in \{1,\ldots,|\cS|/2\}}(x_tx_t^T)\Big]\Big\|_2 \;\; \leq \;\; \sqrt{\frac{32n^2\log(n\ell/\delta)}{|\S|}}.$$
The second term $E^{(2)}$ is a diagonal matrix, with each diagonal entry $E^{(2)}_{ii}$ distributed as 
a binomial distribution. Applying standard Hoeffding's bound, we get that with probability at least $1-\delta$, 
\begin{eqnarray*}
	  \| E^{(2)}\|_2 \;\;=\;\; \max_{i\in[\ell\n]} |E^{(2)}_{ii}|  \;\; \leq \;\; \sqrt{\frac{2\log(2/\delta)}{|\cS|}} \;.
\end{eqnarray*}
This gives the desired bound on $\|E^{(1)}+E^{(2)}\|_2$.

Similarly, $x_{t,i} \|x_t\|_2\leq \sqrt{n}, \forall i$. Hence, using standard Hoeffding Bound, we get with probability at least $1-\delta$, 
$$\Big\|\frac{2}{|\S|}\sum_{t\in [|\cS|/2]}(x_tx_t)_i-\E[S_2]_i\Big\|_2\leq \sqrt{\frac{16n\log(2/\delta)}{|\S|}}.$$
%Note that $\|\sum_{t\in [|\S|]}\E[x^t(x^t)^Tx^t(x^t)^T]\|_2=n\|\sum_{t\in [|\S|]}\E[x^t(x^t)^Tx^t(x^t)^T]\|_2$
%
%=========================================================================
%
\subsection{Proof of Lemma \ref{lem:conc_tr}}
\label{sec:conc_tr}

The claim follows form the following lemma. 

\begin{lemma}
\label{lemma:tls_sample}
Let $H=P_{\Omega_3}(\M_3)\left[\hU_{M_2}\hsm^{-1/2},\hU_{M_2}\hsm^{-1/2},\hU_{M_2}\hsm^{-1/2}\right]$ and $\hH$ be as defined above. Then, with probability larger than $1-\delta$, we have: 
\begin{eqnarray*}
	|H_{abc}-\hH_{abc}| &\leq&2\Big( \frac{2\,r\,n}{\sigma_r(M_2)} \Big)^{3/2} \mu_1^3  \sqrt{\frac{\log(1/\delta)}{|\cS|}} \;.
\end{eqnarray*}
 \end{lemma}
 
 \begin{proof}
   Let $\hH_{abc}=\frac{1}{|\S|}\sum_{t\in \cS} Y^t_{a,b,c}$, where 
   $Y^t_{a,b,c} = \sum_{(i,j,k)\in \Omega_3} x_{t,i}x_{t,j}x_{t,k}  \Qw_{ia} \Qw_{jb} \Qw_{kc},$ where 
   $\hQ=\hum\hsm^{-1/2}$. 
   Then $\E[Y^t]=H$.
   That is, \begin{multline}Y^t_{a,b,c}=\ip{\hQ_a}{x^t}\cdot \ip{\hQ_b}{x^t}\cdot\ip{\hQ_c}{x^t} - \sum_{m\in [n]} \ip{\hQ_a^{(m)}}{(x^t)^{(m)}}\ip{\hQ_b^{(m)}}{(x^t)^{(m)}}\ip{\hQ_c^{(m)}}{(x^t)^{(m)}}\\ - \ip{\hQ_a}{x^t}\cdot\sum_{m\in [n]}\ip{\hQ_b^{(m)}}{(x^t)^{(m)}}\ip{\hQ_c^{(m)}}{(x^t)^{(m)}}-\ip{\hQ_b}{x^t}\cdot\sum_{m\in [n]}\ip{\hQ_a^{(m)}}{(x^t)^{(m)}}\ip{\hQ_c^{(m)}}{(x^t)^{(m)}}\\-\ip{\hQ_c}{x^t}\cdot\sum_{m\in [n]}\ip{\hQ_a^{(m)}}{(x^t)^{(m)}}\ip{\hQ_b^{(m)}}{(x^t)^{(m)}}.\label{eq:yt}\end{multline}
Note that, 
$|\ip{\hQ_b^{(m)}}{x_t^{(m)}}| \leq \frac{\mu_1\sqrt{r}}{\sqrt{n(1-\varepsilon)\sigma_r(\M_2)}}$. 
Hence, for all $a\in[r]$, 
$|\ip{\hQ_a}{x^t}| \leq \frac{\mu_1\sqrt{r \,n}}{\sqrt{(1-\varepsilon)\sigma_r(\M_2)}} $. 
%where the last inequality follows by the Cauchy-Schwarz inequality.  

Using the above inequality with \eqref{eq:yt}, we get: 
$|Y^t_{a,b,c} | \leq \big(r \,n /((1-\varepsilon) \sigma_r(M_2))\big)^{3/2} \mu_1^3 $. 
%Now, $0\leq x^t_a\leq 1$. Also, $\sum_{\lfloor \frac{a}{r}\rfloor r+1 \leq a\leq \lceil \frac{a}{r}\rceil r}x^t_a=1$ Hence, $|Y^t|\leq n\sqrt{n}$. 
 Lemma now follows by using Hoeffding's inequality. 
 \end{proof}

%
%=========================================================================
%
\subsection{Proof of Theorem~\ref{thm:main_comps} }

We first observe that as $U_{M_2}=U R_1$, where $R_1\in \R^{r\times r}$ is an orthonormal matrix. Also, $\Sigma_{M_2}=R_1^T \Sigma  V^T W V \Sigma R_1$. Hence, $\Sigma_{M_2}^{1/2}=R_1^T \Sigma  V^T W^{1/2}R_3,$ where $R_3$ is an orthonormal matrix. Moreover, $\Sigma_{M_2}^{-1/2}=R_3^TW^{-1/2}V \Sigma^{-1} R_1$. Hence,
%Now, $G=M_3[U_{M_2}\Sigma_{M_2}^{-1/2}, U_{M_2}\Sigma_{M_2}^{-1/2}, U_{M_2}\Sigma_{M_2}^{-1/2}]=\sum_{q=1}^k w_q \cdot (R_1^T \Sigma V^T \e_q)\otimes (R_1^T \Sigma V^T \e_q) \otimes (R_1^T \Sigma V^T \e_q)$. Hence, 
\begin{align}G&=M_3[U_{M_2}\Sigma_{M_2}^{-1/2}, U_{M_2}\Sigma_{M_2}^{-1/2}, U_{M_2}\Sigma_{M_2}^{-1/2}]=\sum_{q=1}^k w_q (R_3^TW^{-1/2}\e_q)\otimes (R_3^TW^{-1/2}\e_q)\otimes (R_3^TW^{-1/2}\e_q)\nonumber \\
&=\sum_{q=1}^k \frac{1}{\sqrt{w_q}}(R_3^T\e_q)\otimes (R_3^T\e_q) (R_3^T\e_q).\end{align}
Now, using orthogonal tensor decomposition method of \cite{AGHKT12}, we get: $\Lambda^G=W^{-1/2}$ as the eigenvalues and $V^G=R_3^T$ as the eigenvectors. Theorem now follows by observing:
$$U_{M_2}\cdot \Sigma_{M_2}^{1/2} \cdot V^G \cdot \Lambda^G=U_{M_2}\cdot R_1^T \Sigma   V^T W^{1/2}R_3\cdot R_3^T\cdot W^{-1/2} =U\Sigma V^T=\Pi.$$

%
%=========================================================================
%
\subsection{Proof of Theorem~\ref{thm:consistency} and Theorem~\ref{thm:finite}}
\label{sec:prf_main}
\begin{proof}[Proof of Theorem~\ref{thm:consistency}] 
Recall that in this case, the number of samples are infinite, i.e., $|\S|=\infty$. 
Hence, $P_{\Omega_2}(S_2)=P_{\Omega_2}(M_2)$. That is, $E=0$. Furthermore, $T=\infty$. 
Hence, using Theorem~\ref{thm:matrixam}, Algorithm~\ref{algo:altmin} exactly recovers $M_2$, i.e., $\hM_2^{(T)}=M_2$. 

Furthermore, using Theorem~\ref{thm:tensorls}, we have $\hG=G$; as, $\varepsilon=\|M_2-\hM_2\|_2 = 0$ and $|\S|=\infty$. 
Now, consider $R_3R_3^T=\hsm^{-1/2}\hum^T\Pi W^{1/2}\cdot W^{1/2}\Pi^T\hum\hsm^{-1/2}=\hsm^{-1/2}\hum^TM_2\hum\hsm^{-1/2}=I$. That is, $R_3$ is orthonormal. Hence, using orthogonal decomposition method of \cite{AGHKT12} (see Theorem~\ref{thm:aghkt}), we get $V^G=R_3$ and $\Lambda^G=W^{-1/2}$. Now, using step 6 of Algorithm~\ref{algo:main}, $\hPi=\hum\hum^T\Pi$. Theorem now follows as $\hum\hum^TU=U$ using Remark~\ref{claim:tls_cl1}. 

Also note that 
from Theorem \ref{thm:matrixam}, 
$\hM_2$ is $\mu_1$ incoherent with $\mu_1=6\mu\sigma_1(M_2)/\sigma_r(M_2)$.
%Further, Let $\r$ denote the rank of $M_2$, then 
%the incoherence $\mu(M_2)$ is upper bounded by 
%\begin{eqnarray}
%	\mu(\M_2) \leq \sqrt{\frac{n}{ \r\,\sigma_\r(\M_2)}} \;.
%	\label{eq:M2incoherence}
%\end{eqnarray}
%Recall that $\M_2=\Prb \Wt \Prb^T$ and let the singular value decomposition be $\M_2 = U\Sigma U^T$. 
%Then, $U = \Prb W^{1/2} Q \Sigma^{-1/2}$ for some orthogonal matrix $Q\in\reals^{\k\times r}$, 
%where $r$ is the rank of $\M_2$. 
%For the $i$-th block, we have $
%\|U_i\|_2 = \| (\prb^{(i)})^T \Wt^{1/2} Q \Sigma^{-1/2} \|_2   \leq  \|(\prb^{(i)})^T \Wt^{1/2} \|_2 \,\|Q\Sigma^{-1/2}\|_2 \leq 1/{\sqrt{\sigma_r(\M_2)}} \;,$ 
%where we used the fact that $\pi^{(i)}$ is a row-stochastic matrix  
%and hence 
%$\|(\pi^{(i)})^T\Wt^{1/2}\|_2 \leq 1$. 
%For a  row stochastic matrix with $\sum_{b\in[\ell]} (\pi^{(i)}_{a,b})^2 \leq 1 $, this follows from 
%\begin{eqnarray*}
%\|(\pi^{(i)})^T\Wt^{1/2}\|_2^2 &=& \max_{\|x\|=1}  \|(\pi^{(i)})^T\Wt^{1/2} x \|^2 \\
%&=&\max_{\|x\|=1}  \sum_{a\in [\ell]}\big(\sum_{b\in[k]} \pi^{(i)}_{b,a} w_b^{1/2}x_b\big)^2 \\
%&\leq& \max_{\|x\|=1}  \sum_{a\in [\ell]}\big(\sum_{b\in[k]} (\pi^{(i)}_{b,a} w_b^{1/2})^2 \|x\|^2\big) \\
%&=& \sum_{b\in[k]} \big(  w_b \sum_{a\in [\ell]}(\pi^{(i)}_{b,a})^2 \big) \;\leq\;1 \;.
%\end{eqnarray*}

\end{proof}
%
%=========================================================================
%
\begin{proof}[Proof of Theorem~\ref{thm:finite}]

To simplify the notations, we will assume that the permutation that matches the output of our algorithm to the actual 
types is the identity permutation. 
Let's define 
\begin{eqnarray}
	\varepsilon_M \;\equiv\;\frac{\|\hM_2-M_2\|_2}{\sigma_r(M_2)}\; \; \text{ and } \;\;\;\; \varepsilon_G\;\equiv\; \|\hG-\tG\|_2\;, 
	\label{eq:main_2}
\end{eqnarray}
where $\hG$ is the output of the \tensorLS{} and $\tG=M_3[\hU_{M_2}\hSigma_{M_2},\hU_{M_2}\hSigma_{M_2},\hU_{M_2}\hSigma_{M_2}]$.  

The spectral algorithm outputs $\hPi = \hU_{M_2}\hSigma_{M_2}^{1/2}\hV^G\hLambda^G$, 
and we know that $\Pi=U_{M_2}\Sigma^{1/2}_{M_2}V^GW^{-1/2}$.
In order to show that these two matrices are close, now might hope to prove that each of the terms are close. 
For example we want $\|U_{M_2}-\hU_{M_2}\|_2$ to be small. However, even if $U_{M_2}$ and $\hU_{M_2}$ span the same 
subspaces the distance might be quite large. Hence, we project $P$ onto the subspace spanned by $\hU_{M_2}$ to prove the bound we want. Define 
\begin{eqnarray}
	\tV &\equiv& \hSigma_{M_2}^{-1/2} \hU_{M_2}^T \Pi W^{1/2}\;,
\end{eqnarray}
such that 
\begin{eqnarray}
	\tG &=& \sum_{i=1}^r \frac{1}{\sqrt{w_i}} (\tv_i \otimes \tv_i \otimes \tv_i) \;,
\end{eqnarray}
where $\tV=[\tv_1,\ldots,\tv_r]$. Then, we have $\hU_{M_2}\Pi = \hSigma_{M_2}^{1/2}\tV Q^{-1/2}$.
Then, 
\begin{eqnarray}
	\|\Pi-\hPi\|_2 &\leq& \|\hU_{M_2}\hU_{M_2}^T \Pi -\Pi \|_2 + \| \hPi - \hU_{M_2}\hU_{M_2}^T \Pi\|_2 \nonumber\\
		&=& \|(\hU_{M_2}\hU_{M_2}^T-\id) \Pi \|_2 + \| \hU_{M_2}\hSigma_{M_2}^{1/2}\hV^G \hLambda^G - \hU_{M_2}  \hSigma^{1/2}_{M_2} \tV^G W^{-1/2}\|_2\nonumber\\
		&\leq& \|(\hU_{M_2}\hU_{M_2}^T-\id)  \Pi \|_2 + \| \hU_{M_2}\hSigma_{M_2}^{1/2}(\hV^G -\tV )W^{-1/2}\|_2 + \| \hU_{M_2}\hSigma_{M_2}^{1/2}\hV^G (\hLambda^G - W^{-1/2})\|_2\;.\label{eq:errorbound0}
\end{eqnarray}

To bound the first term, denote the SVD of $\Pi$ as $\Pi=U\Sigma V^T$.
Using Remark~\ref{claim:tls_cl1}, $\|\hum\hum^T\Pi-\Pi\|_2\leq \|\hum\hum^TU-U\|_2\|\Sigma\|_2\leq \varepsilon_M\sigma_1(\Pi)$. 

Note that $\|\hSigma_{M_2}\|_2 \leq \|\hM_2-M_2\|+ \|M_2\|_2  \leq \varepsilon_{M}\sigma_r(M_2) + \|M_2\|_2 \leq 2\|M_2\|_2$, 
when $\varepsilon_M\leq 1/2$.
To prove that the second term is bounded by $C\sqrt{\|M_2\|_2 r \wmax/\wmin} (\varepsilon_G + (1/\sqrt{\wmin})\varepsilon_M )$, 
we claim that 
\begin{eqnarray*}
	\|\tV - \hV^G\|_2 &\leq &  C \sqrt{r \wmax} \Big(\varepsilon_G + \frac{1}{\sqrt{\wmin}}\varepsilon_M \Big)\;, \text{ and }\\
	\| W^{-1/2}-\hLambda^G\|_2	&\leq& C\, \Big(\varepsilon_G + \frac{1}{\sqrt{\wmin}}\varepsilon_M \Big)\;.
\end{eqnarray*}

Now recall that, $R_3=\hsm^{-1/2}\hum^T \Pi W^{1/2}$. 
Let the SVD of $\tV$ be $\tV= U_1 \Sigma_1 V_1^T$. 
Define an orthogonal matrix $R=U_1 V_1^T$, such that $RR^T=R^TR=\id$. 
Using Remark~\ref{claim:tls_cl2} we have $\|\tV-R\|_2\leq 2\varepsilon_M$. 
Moreover, $\tG=\sum_{q\in [r]} \frac{1}{\sqrt{w_q}}(R\e_q\otimes R\e_q\otimes R\e_q)+E_G$, where \begin{equation}\label{eq:main_eg}\|E_G\|_2\leq 2\frac{\varepsilon_M(1+\varepsilon_M)^2}{\sqrt{w_{min}}}\leq \frac{8\varepsilon_M}{\sqrt{w_{min}}}, \end{equation}
where, last inequality follows by $\varepsilon_M\leq 1$. 

Hence, using \eqref{eq:main_2}, \eqref{eq:main_eg}, we have (w.p. $\geq 1-2\delta$): 
\begin{equation}
  \label{eq:main_ehg}
\|\hG-\sum_{q\in [r]} \frac{1}{\sqrt{w_q}}(R\e_q\otimes R\e_q\otimes R\e_q)\|_2\leq \varepsilon_G+\|E_G\|_2\leq \varepsilon_G + (8/\sqrt{\wmin})\varepsilon_M. 
\end{equation}
Since $R$ is orthogonal by construction, we can apply 
Theorem~\ref{thm:aghkt} to bound the distance between 
$\hV^G$ and $R$, i.e. $\|\hV^G-R\|_2 \leq 
8\sqrt{r\,\wmax }(\varepsilon_G+(8/\sqrt{\wmin})\varepsilon_M)$. 
By triangular inequality, we get that 
\begin{eqnarray*}
	\|\hV^G - \tV\|_2 &\leq& \|\hV^G-R\|_2 + \|R-\tV\|_2 \\
		&\leq& 8\sqrt{r\,\wmax }\,\Big(\,\varepsilon_G+\frac{8}{\sqrt{\wmin}}\varepsilon_M\,\Big)\, + 2\varepsilon_M \\
		&\leq& C\sqrt{r\,\wmax }\,\Big(\,\varepsilon_G+\frac{1}{\sqrt{\wmin}}\varepsilon_M\,\Big)  \;.
\end{eqnarray*}  
Similarly, 
\begin{equation*}
	\|W^{-1/2}-\hLambda^G\|_2 \;\leq\; 5\,\Big(\varepsilon_G+ \frac{8}{\sqrt{\wmin}}\varepsilon_M\,\Big)\;. 
\end{equation*}
This implies that the third term in \eqref{eq:errorbound0} is bounded by 
$\|\hU_{M_2} \hSigma_{M_2}^{1/2}\hV^G(\hLambda^G-W^{-1/2})\|_2  \leq C \sqrt{\|M_2\|_2} (\varepsilon_G + \varepsilon_M/\sqrt{\wmin})$, using the assumption on $|\cS|$ such that $(\sqrt{r\wmax} )\varepsilon_G \leq C$ and $(\sqrt{r \wmax/\wmin} )\varepsilon_M \leq C$. 

Putting these bounds together, we get that 
\begin{eqnarray*}
	\|\hPi - \Pi\|_2 &\leq& C \sqrt{\frac{r\,\wmax\,\|M_2\|_2}{\wmin}}\Big(\,\varepsilon_G+\frac{1}{\sqrt{\wmin}}\varepsilon_M \,\Big)\;,
\end{eqnarray*}
where we used the fact that $\|\Pi\|_2 \leq (1/\sqrt{\wmin}) \|M_2\|_2^{1/2}$.

From Theorems \ref{thm:matrixam} and \ref{thm:tensorls} and Lemmas \ref{lem:conc_mx} and \ref{lem:conc_tr}, we get that 
\begin{eqnarray*}
	\varepsilon_M &\leq&   C\frac{n\,\|M_2\|_F\,r^{1/2}}{\sigma_r(M_2)^2} \sqrt{\frac{\log(n/\delta)}{|\S|}} \;,\;\text{ and  }\\
	 \varepsilon_G &\leq& C \frac{\mu^4 r^{3.5}}{\sqrt{\wmin}}\Big(\frac{\sigma_1(\M_2)}{\sigma_r(\M_2)}\Big)^{4.5} \frac{1}{n} \varepsilon_M + Cr^3\mu^3\frac{\sigma_1(M_2)^3\,n^{1.5}}{\sigma_r(M_2)^{4.5}} \sqrt{\frac{\log(n/\delta)}{|\S|}}\;,
\end{eqnarray*}
when $|\S| \geq C' (\ell+r)(n^2/\sigma_r(M_2)^2)\log(n/\delta)$ and $n\geq C'(r^3+r^{1.5}\mu^2)(\sigma_1(M_2)/\sigma_r(M_2))^2$.
Further, if $n \geq C' \mu^4r^{3.5}(\sigma_1(M_2)/\sigma_r(M_2))^{4.5}$, then 
\begin{eqnarray*}
	\varepsilon_G &\leq& C \frac{1}{\sqrt{\wmin}} \varepsilon_M + Cr^3\mu^3\frac{\sigma_1(M_2)^3\,n^{1.5}}{\sigma_r(M_2)^{4.5}} \sqrt{\frac{\log(n/\delta)}{|\S|}} \;.
\end{eqnarray*}

\end{proof}

\begin{theorem}[Restatement of Theorem 5.1 by \cite{AGHKT12}]\label{thm:aghkt}
  Let $G=\sum_{i \in [r]} \lambda_i (v_i \otimes v_i \otimes v_i)+E$, where $\|E\|_2\leq C_1\frac{\lambda_{min}}{r}$. Then the tensor power-method after $N\geq C_2 (\log r+ \log \log \left(\frac{\lambda_{max}}{\|E\|_2}\right)$, generates vectors $\hv_i, 1\leq i\leq r$, and $\hlambda_i, 1\leq i\leq r$, s.t., 
  \begin{equation}
    \label{eq:thm_tnsr}
    \|v_i-\hv_{P(i)}\|_2\leq 8\|E\|_2/\lambda_{P(i)},\quad |\lambda_{i}-\hlambda_{P(i)}|\leq 5\|E\|_2. 
  \end{equation}
where $P$ is some permutation on $[r]$. 
\end{theorem}

%
%=========================================================================
%
\subsection{Proof of Corollary \ref{coro:KL}}
\label{sec:KL} 

Feldman et al. proved that if we have a good estimate of $w_i$'s and $\pi_i$'s in absolute difference, then  
the thresholding and normalization defined in Section \ref{sec:result} gives a good estimate in KL-divergence. 
\begin{thm}[{\cite[Theorem 12]{FOS08}}]
	Assume $\Z$ is a mixture of $\k$ product distributions on $\{1,\ldots,\ell\}^n$ with mixing weights $w_1\ldots,w_\k$ and 
	probabilities $\pi^{j}_{i,a}$, and the following are satisfied: 
	\begin{itemize}
		\item for all $i\in[\k]$ we have $|w_i-\hw_i|\leq\varepsilon_w$, and
		\item for all $i\in[\k]$ such that $w_i\geq \varepsilon_{\rm min}$ we have $|\pi^{(j)}_{i,a}-\hpi^{(j)}_{i,a}| \leq \varepsilon_\pi$ for all $j\in[\n]$ and $a\in[\ell]$.
	\end{itemize}
	Then, for sufficiently small $\varepsilon_w$ and $\varepsilon_\pi$, the mixture $\hZ$ satisfies 
	\begin{eqnarray}
		D_{\rm KL}(\Z||\hZ) &\leq& 12 n \ell^3 \varepsilon_\pi^{1/2} + n k \varepsilon_{\rm min} \log(\ell/\varepsilon_\pi) + \varepsilon_w^{1/3}\;. \label{eq:KLbound}
	\end{eqnarray}
\end{thm}
For the right-hand-side of \eqref{eq:KLbound} to be less than $\eta$, 
it suffices to have $\varepsilon_w=O(\eta^3)$, $\varepsilon_\pi=O(\eta^2/n^2\ell^6)$, and 
$\varepsilon_{\rm min}=O(\eta/nk\log(\ell/\varepsilon_\pi))$. 

From Theorem \ref{thm:finite}, 
$|\hw_i-w_i| = O(\varepsilon_M ) $. 
Then $\varepsilon_M \leq C \eta^3  $ 
for some positive constant $C$ ensures that 
the condition is satisfied with $\varepsilon_w=O(\eta^3)$.
From Theorem \ref{thm:finite},  we know that that 
$|\hpi^{(j)}_{i,a} - \pi^{(j)}_{i,a}| =O(\varepsilon_M \sqrt{\sigma_1(M_2)\wmax r/ \wmin})$.
Then $\varepsilon_M \leq C( \eta^2\wmin^{1/2}\,/\,(n^2\ell^6(\sigma_1(M_2) \wmax r)^{1/2}) )$ 
for some positive constant $C$  
ensures that the condition is satisfied with $\varepsilon_\pi = O(\eta^2/n^2\ell^6)$.

These results are true for any values of $w_{\rm min}$, as long as it is positive. 
Hence, we have $\varepsilon_{\rm min}=0$.
It follows that for a choice of 
\begin{eqnarray*}
	\varepsilon_M \leq C\, \eta^2  \,\min\Big\{ \frac{\wmin^{1/2}}{ n^2\ell^6 (\sigma_1(M_2) \, \wmax \,r)^{1/2}} , \eta \Big\} \;,
\end{eqnarray*}
we have the desired bound on the KL-divergence. 

%
%=========================================================================
%
\subsection{Proof of Corollary \ref{coro:cluster}}
\label{sec:cluster}

We use a technique similar to those used to analyze distance based clustering algorithms in \cite{AK01,AM05,McS01}. 
The clustering algorithm of  \cite{AK01} 
uses $\hPi$ obtained in Algorithm~\ref{algo:main} 
to reduce the dimension of the samples and apply distance based clustering algorithm of \cite{AK01}. 

%\begin{algorithm}[t!]
%\caption{Distance based clustering approach}
%\label{algo:cluster}
%\begin{algorithmic}[1]
%\STATE Input: $S_2=\frac{2}{|\S|}\sum_{t\in \{1,\ldots, |\cS|/2\} }x_tx_t^T$, $\Omega_2$, $r$, $T$
%\STATE Initialize $\ell\n\times\r$ dimensional matrix $U_0\leftarrow $ top-$r$ eigenvectors of $P_{\Omega_3}(S_2)$
%\FORALL{$\tau=1 $ to $T-1$}
%\STATE $\hU_{\tau+1}=\arg\min_{U}\|\cP_{\Omega_2}(S_2)-\cP_{\Omega_2}(UU_\tau^T)\|_F^2$
%\STATE $[U_{\tau+1} R_{\tau+1}]={\rm QR}(\hU_{t+1})$ (QR decomposition)
%\ENDFOR
%\STATE Output: $\hM_2=(\hU_{T})(U_{T-1})^T$
%\end{algorithmic}
%\end{algorithm}

Following the analysis of \cite{AK01}, 
we want to identify the conditions such that 
two samples from the same type are closer than the distance between two samples from two different types. 
In order to get a large enough gap, we apply $\hPrb$ and show that 
\begin{eqnarray*}
	\|\hPrb^T(x_i-x_j) \|\; < \; \|\hPrb^T(x_i-x_k)\|\;,
\end{eqnarray*}
for all $x_i$ and $x_j$ that belong to the same type and for all $x_k$ with a different type. 
Then, it is sufficient to show that 
$\|\hPrb^T(\pi_a-\pi_b)\| \geq 4 \max_{i\in\cS} \| x_i - \E[x_i]\|$ for all $a\neq b \in[\k]$. 
From Theorem \ref{thm:finite}, we know that for 
$|\cS| \geq C \mu^6 r^7 n^3 \sigma_1(M_2)^7 \wmax \log(n/\delta)/ (\wmin^2 \sigma_r(M_2)^9 \tepsilon^2)$, 
$\|\pi_a-\hpi_a\| \leq \varepsilon_M \sqrt{r \wmax \sigma_1(M_2)/\wmin}\leq \tepsilon$ for all $a\in[\k]$. Then, 
\begin{eqnarray*}
	\|\hPi^T(\pi_a-\pi_b)\| &\geq& \|\Pi^T(\pi_a-\pi_b)\| - \|(\Pi-\hPi)^T(\pi_a-\pi_b)\| \\
		&\geq&  \sqrt{(\pi_a^T(\pi_a-\pi_b))^2 + (\pi_b^T(\pi_a-\pi_b))^2} - \|\Pi-\hPi\|_2\, \|\pi_a-\pi_b\| \\
		&\geq& \|\pi_a-\pi_b\|^2  - \sqrt{\k} \tepsilon \|\pi_a-\pi_b\| 
\end{eqnarray*}
On the other hand, applying a  concentration of measure inequality gives 
\begin{eqnarray*}
	\prob \Big( |\hpi_a^T(x_i-\E[x_i])| \geq \|\hpi_a\| \sqrt{2\log(\k/\delta)} \Big) &\leq& \frac{\delta}{\k}\;.
\end{eqnarray*}
Applying union bound, 
$\|\hPi^T(x_i-\E[x_i])\| \leq \|\hPi\|_F \, \sqrt{2\log(\k/\delta)} \leq (\sqrt{2}\,\|\Pi\|_F + \sqrt{2\k}  \tepsilon ) \sqrt{4 \log(\k/\delta)}$ 
with probability at least $1-\delta$, 
where we used the fact that 
$\|\hPi\|_F^2 \leq \sum_a (\|\pi_a\|+\tepsilon)^2 \leq 2\sum_a (\|\pi_a\|^2+\tepsilon^2)\leq 
2(\|\Pi\|_F +\sqrt{\k}\,\tepsilon)^2$.

For $\tepsilon \geq (\|\pi_a-\pi_b\|^2 - \|\Pi\|_F \sqrt{8 \log(r/\delta)})/(\sqrt{r}\|\pi_a-\pi_b\| + \sqrt{8 r \log(r/\delta)})$, 
it follows that 
$\|\hPrb^T(\pi_a-\pi_b)\| \geq 4 \max_{i\in\cS} \| x_i - \E[x_i]\|$, and this proves that the distance based algorithm of 
\cite{AK01} will succeed in finding the right clusters for all samples. 

%%% Local Variables: 
%%% mode: latex
%%% TeX-master: "crowd"
%%% End: 

%\input{proofs}

\end{document}